%% file: aaai25.tex
\documentclass[letterpaper]{article} 
\usepackage{aaai25}  
\usepackage{times}  
\usepackage{helvet}  
\usepackage{courier}  
\usepackage[hyphens]{url}  
\usepackage{graphicx} 
\usepackage{natbib}  
\usepackage{caption} 
\usepackage{algorithm}
\usepackage{algorithmic}

\usepackage{newfloat}
\usepackage{listings}
\DeclareCaptionStyle{ruled}{labelfont=normalfont,labelsep=colon,strut=off} 
\floatstyle{ruled}
\newfloat{listing}{tb}{lst}{}
\floatname{listing}{Listing}
\usepackage{subcaption}
\usepackage{amsthm}
\usepackage{amsmath,amsfonts,bm, dsfont}
\usepackage{multirow}
\usepackage{array}
\usepackage{longtable} 

\newtheorem{proposition}{Proposition}
\newtheorem{lemma}{Lemma}
\newtheorem{definition}{Definition}

\usepackage{booktabs}

\def\rvb{{\mathbf{b}}}
\def\rvx{{\mathbf{x}}}
\def\rvy{{\mathbf{y}}}
\def\rvw{{\mathbf{w}}}
\def\gN{{\mathcal{N}}}
\def\gU{{\mathcal{U}}}
\def\mI{{\bm{I}}}
\def\vzero{{\bm{0}}}
\def\1{\mathds{1}}

\def\gB{{\mathcal{B}}}
\def\gD{{\mathcal{D}}}
\newcommand{\R}{\mathbb{R}}
\newcommand{\norm}[1]{\left\lVert#1\right\rVert}
\newcommand{\meanp}[2]{\mathbb{E}_{#1} \left\lbrack #2 \right\rbrack}
\newcommand{\kl}[2]{\mathrm{KL}\left(#1 || #2\right)}
\DeclareMathOperator*{\argmin}{arg\,min}
\DeclareMathOperator{\Tr}{Tr}
\newcommand{\Var}{\mathrm{Var}}
\def\transpose{{\top}}
\newcommand{\E}{\mathbb{E}}

\def\method{MBM}
\def\score{\texttt{s}}
\def\bridge{\rvb}
\def\distancefn{\ell}

\usepackage{newfloat}
\usepackage{listings}
\DeclareCaptionStyle{ruled}{labelfont=normalfont,labelsep=colon,strut=off} 
\floatstyle{ruled}
\newfloat{listing}{tb}{lst}{}
\floatname{listing}{Listing}
\usepackage{xcolor}
\usepackage[colorlinks]{hyperref}

\def\tmp#1#2#3{%
  \definecolor{Hy#1color}{#2}{#3}%
  \hypersetup{#1color=Hy#1color}}
\tmp{link}{HTML}{800006}
\tmp{cite}{HTML}{2E7E2A}
\tmp{file}{HTML}{131877}
\tmp{url} {HTML}{8A0087}
\tmp{menu}{HTML}{727500}
\tmp{run} {HTML}{137776}
\def\tmp#1#2{%
  \colorlet{Hy#1bordercolor}{Hy#1color#2}%
  \hypersetup{#1bordercolor=Hy#1bordercolor}}
\tmp{link}{!60!white}
\tmp{cite}{!60!white}
\tmp{file}{!60!white}
\tmp{url} {!60!white}
\tmp{menu}{!60!white}
\tmp{run} {!60!white}

\usepackage[capitalise]{cleveref}
\usepackage{todonotes}
\usepackage{amsthm}

\title{Constrained Generative Modeling with \\Manually Bridged Diffusion Models}
\author{
    Saeid Naderiparizi\equalcontrib \textsuperscript{\rm 1,3},
    Xiaoxuan Liang\equalcontrib \textsuperscript{\rm 1,3},
    Berend Zwartsenberg\textsuperscript{\rm 3},
    Frank Wood\textsuperscript{\rm 1,2,3}
}
\affiliations{
    \textsuperscript{\rm 1}Department of Computer Science, University of British Columbia, Vancouver, Canada\\
    \textsuperscript{\rm 2}Alberta Machine Intelligence Institute (Amii), Edmonton, Canada\\
    \textsuperscript{\rm 3}InvertedAI, Vancouver, Canada\\
    saeidnp@cs.ubc.ca, liang51@cs.ubc.ca, berend.zwartsenberg@inverted.ai, fwood@cs.ubc.ca
}

\begin{document}

\maketitle

\begin{abstract}
In this paper we describe a novel framework for diffusion-based generative modeling on constrained spaces. In particular, we introduce manual bridges, a framework that expands the kinds of constraints that can be practically used to form so-called diffusion bridges. We develop a mechanism for combining multiple such constraints so that the resulting multiply-constrained model remains a manual bridge that respects all constraints. We also develop a mechanism for training a diffusion model that respects such multiple constraints while also adapting it to match a data distribution. We develop and extend theory demonstrating the mathematical validity of our mechanisms. Additionally, we demonstrate our mechanism in constrained generative modeling tasks, highlighting a particular high-value application in modeling trajectory initializations for path planning and control in autonomous vehicles.
\end{abstract}
\begin{links}
\link{Code}{github.com/plai-group/manually-bridged-models}
\end{links}

\section{Introduction}

For generative models to become practically useful in embodied artificial intelligence domains, guaranteed constraint satisfaction is effectively required.  Examples abound, such as path planning and control in autonomous vehicles and advanced driver-assistance (AV/ADAS) systems \cite{janner2022planning, zhong2022guided}, kinematic, dynamics, power, and other constraints in robotics \cite{schulman2014motion}, safety critical plant operation \cite{knight2002safety}, etc.  Strictly eliminating non-factual or offensive hallucinations from large language models (LLM) falls in this problem category too \cite{azamfirei2023large, huang2023survey}.

Our experimental focus in this paper will be on a particular subproblem in AV/ADAS planning and behavioral simulation, so starting here, we will motivate our work using language from this domain; however, note that the mechanisms and theory we develop are general.

\begin{figure*}[t]
    \centering
   \begin{subfigure}{0.24\textwidth}
       \includegraphics[width=\textwidth]{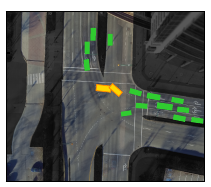}
       \caption{Standard diffusion\hfill}
       \label{fig:banner-baseline}
   \end{subfigure}\hfill
   \begin{subfigure}{0.24\textwidth}
       \includegraphics[width=\textwidth]{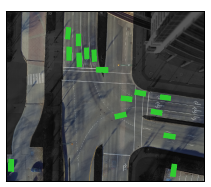}
       \caption{Conditional diffusion\hfill}
       \label{fig:banner-conditional}
   \end{subfigure}\hfill
    \begin{subfigure}{0.24\textwidth}
        \captionsetup{justification=centering}
        \includegraphics[width=\textwidth]{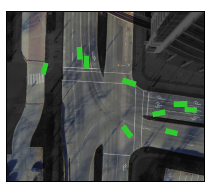}
        \caption{Manual bridge (\texttt{DB-arch})}
        \label{fig:banner-sde}
   \end{subfigure}\hfill
   \begin{subfigure}{0.24\textwidth}
        \captionsetup{justification=centering}
        \includegraphics[width=\textwidth]{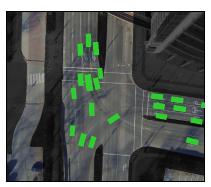}
        \caption{\method{}}
        \label{fig:banner-sde-inp}
    \end{subfigure}\hfill%
    \begin{subfigure}{0.24\textwidth}
        \includegraphics[width=\textwidth]{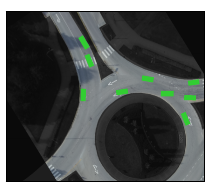}
        \caption{}
        \label{fig:banner-e}
    \end{subfigure}\hfill
    \begin{subfigure}{0.24\textwidth}
        \includegraphics[width=\textwidth]{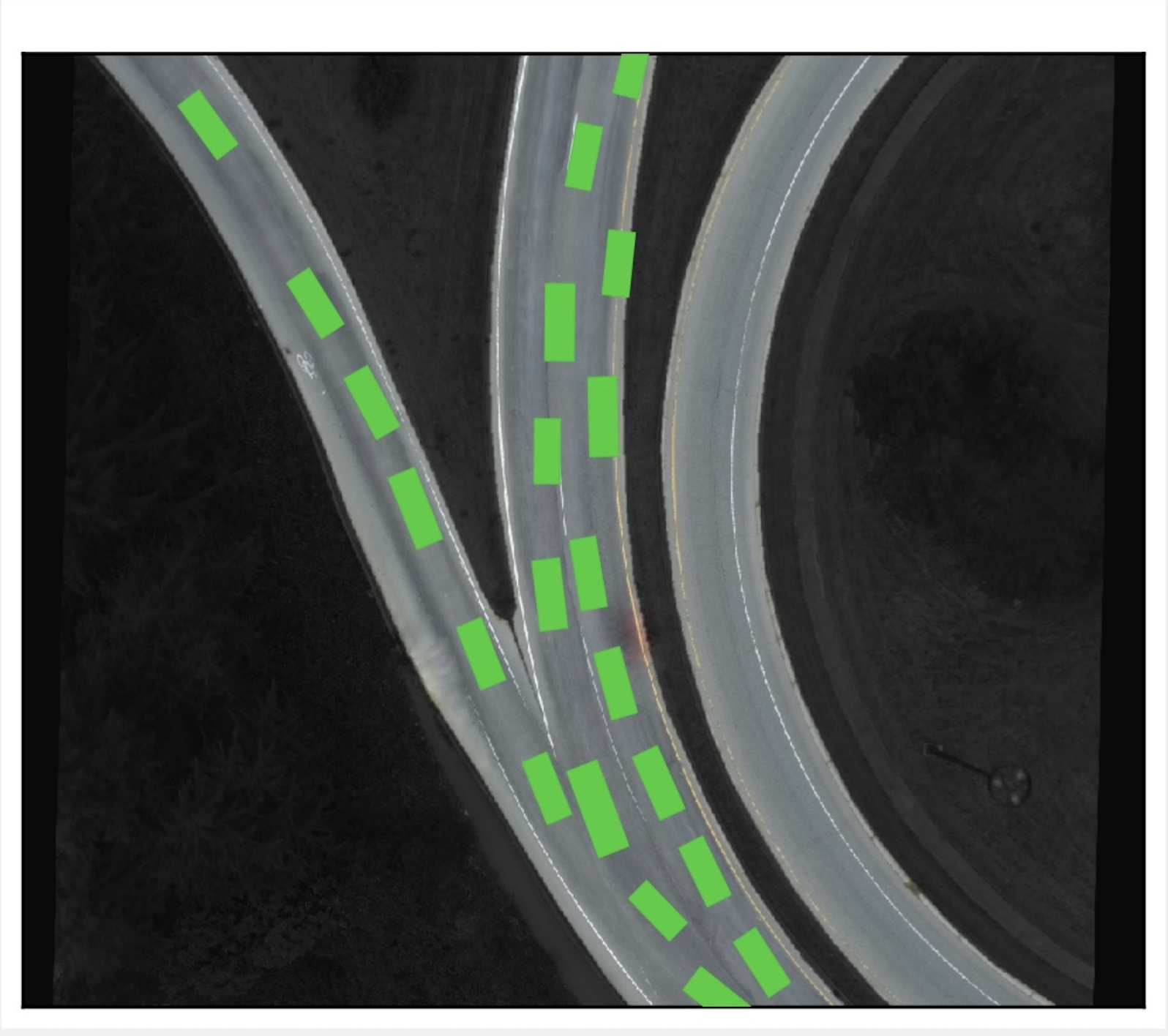}
        \caption{}
        \label{fig:banner-f}
    \end{subfigure}\hfill
    \begin{subfigure}{0.24\textwidth}
       \includegraphics[width=\textwidth]{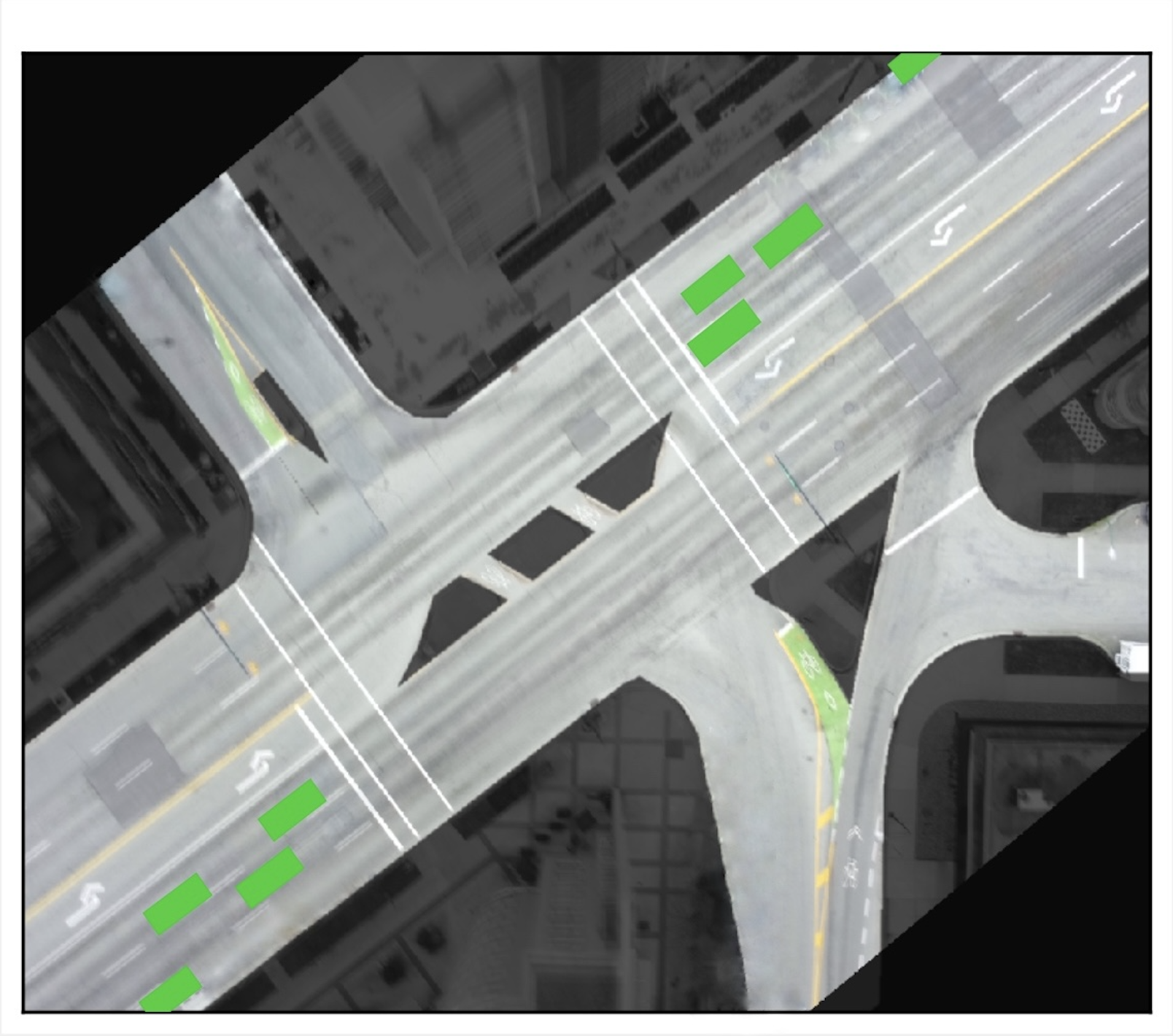}
        \caption{}
       \label{fig:banner-g}
    \end{subfigure}\hfill
    \begin{subfigure}{0.24\textwidth}
        \includegraphics[width=\textwidth]{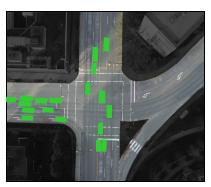}
        \caption{}
        \label{fig:banner-h}
    \end{subfigure}
    \caption{\method{} applied to traffic scene generation. The goal is to place  vehicles on a given bird's-eye view image of a map; indicated here as the ``light'' region of an underlying aerial image. The model output is the set of ``cars'' (infraction-free cars are  green; cars involved in infractions are yellow). The top row shows samples from different models given the same map.  The standard diffusion sample (\subref{fig:banner-baseline}) contains a collision infraction.
     The rest of the top row shows different architectural mechanisms to avoid infractions.  Both conditional diffusion 
     (\subref{fig:banner-conditional}) and (\subref{fig:banner-sde}) are not realistic:
    they both distort the distribution, albeit in different ways, this effect being more apparent in (\subref{fig:banner-sde}). A sample from \method{} in (\subref{fig:banner-sde-inp}) shows no infractions while remaining realistic. The second row shows samples from \method{} on additional maps.}
    \label{fig:banner}
\end{figure*}

Consider the problem of realistically distributing agents in top-down, two-dimensional space; cars, pedestrians, etc.  There are several characteristics of such a distribution that are sufficiently close to being constraints that they may as well be.   Cars and other vehicles are both constrained to be ``on road'' and also to be not overlapping (colliding).  Existing generative models fit to even very large datasets of such data struggle in the sense that samples from them exhibit ``infractions'' (constraint violations) at excessive rates \cite{zwartsenberg2022conditional,jiang2024scenediffuser,niedoba2024diffusion}; even when the data on which they are trained is cleaned to contain only non-infracting examples.

Why does this occur?  In the non-parametric limit this problem would not exist. However, in any finite-data and finite-capacity model the particular generalization strategy the model employs remains a degree of freedom.  We are motivated to seek modeling approaches that allow us to direct and exclude generalizations that place mass outside whatever problem specific constraint set there is.

Various approaches to this have been developed \cite{chang2024safe, huang2024versatile}; this paper explores diffusion-bridge-like mechanisms, inspired by diffusion bridge theory, that empirically demonstrate superior constraint-satisfaction at little expense to generalization otherwise.  We call these mechanisms ``manual bridges'' to distinguish them from the more formally mathematically delimited bridge functions employed in the greater diffusion bridge literature \cite{schauer2017guided,de2021diffusion,wang2021deep,heng2021simulating,chen2021likelihood,zhou2024denoising,shi2024diffusion}.

Our central contribution is an architecture that allows both for imposition of so called ``manual bridges'' to impose constraints in diffusion-based generative models and stable training of models that are constrained in this way, resulting in ``manually bridged models,'' a novel family of generative models that are capable of fitting complex distributions well while also respecting and representing the sharp boundaries imposed by constraints.

\section{Related Work}
A framework for constraints on the sample space of generative models was first theoretically described in \cite{hanneke2018actively}. In their framework, constraints are represented through a black-box oracle function labeling samples as valid or invalid. This problem has been practically explored for generative adversarial networks (GANs) \cite{kong2023data} and diffusion models \cite{naderiparizi2024dont}.
More recently, \citet{christopher2024constrained} proposed a method for generating constraint satisfying samples from pre-trained diffusion models. It requires a projection operator the constraint set, which is generally intractable for complex constraints. Moreover, they employ a Langevin dynamics-based \citep{welling2011bayesian} sampler which is slow to converge.
The main idea in all of these methods was to improve generative models by incorporating information from the constraints. In this paper, however, the goal is to construct a model family that does not generate invalid samples by design.

Recognizing the expressivity of diffusion models, various approaches to incorporating pre-defined constrains into them have emerged in the literature. \citet{lou2023reflected} proposed reflected diffusion models that enforce the whole diffusion sampling trajectory to remain bounded in a convex set. \citet{fishman2023diffusion,fishman2024metropolis} extended reflected diffusion models to support more general constraints. However, their approaches are only evaluated on low-dimensional problems with simple constraints. Moreover, reflected diffusion makes the forward process, and consequently training, expensive.  \citet{fishman2023diffusion} also proposed a barrier function based approach for constrained diffusion models. \citet{liu2024mirror} used barrier functions to transform constrained domains into unconstrained dual ones. Both these methods only support convex constraints.

Another closely related body of work is diffusion bridges, stochastic processes that are guaranteed to end in a given constraint set. \citet{wu2022diffusion} developed a set of mathematically sufficient conditions for designing diffusion bridges to a given constraint set.  The follow-up work of \citet{liu2023learning} used bridges to formulate diffusion models on discrete sets. They also provided closed-form bridges for a restricted set of constraints such as product of intervals. These closed form bridges quickly become intractable as the constraint set gets more complex.

\section{Background}
\subsection{Diffusion Models}\label{sec:background:dm}
Diffusion models \cite{sohl2015deep, song2019generative, song2020score} are a class of generative models that learn to invert a stochastic process, known as the ``forward process,'' that gradually adds noise to samples from a data distribution $q_0(\rvx_0)$. The forward process is formulated as an SDE:
\begin{equation}
    d\rvx_t = f(\rvx_t; t) dt + g(t) d \rvw, \qquad \rvx_0 \sim q_0,
    \label{eq:forward-process}
\end{equation}
where $f$ and $g$ are drift and diffusion functions and $\rvw$ is the standard Wiener process. The forward process is designed such that the SDE's solution at time $T$ is $q_T(\rvx_T) \approx \pi(\rvx_T)$ for some known $\pi$ typically equal to $\gN(\vzero, \mI)$.

\citet{anderson1982reverse} showed that the path measure on the continuous trajectories following the forward process in \cref{eq:forward-process} is identical to the one governed by the following time-reversed SDE:
\begin{equation}
    d\rvx_t = [f(\rvx_t; t) - g^2(t) \nabla_x \log q_t(\rvx_t)]\,dt + g(t)\,d\bar{\rvw},
    \label{eq:reverse-process}
\end{equation}
where $\rvx_T \sim \pi$ and $\bar{\rvw}$ is the Wiener process when time flows backwards. Diffusion models learn the score function $\score_\theta(\rvx_t;t) \approx \nabla_x\log q_t(\rvx_t)$ and approximate the reverse process in \cref{eq:reverse-process} by:
\begin{equation}
    d\rvx_t = [f(\rvx_t; t) - g^2(t) \score_\theta(\rvx_t; t)]\,dt + g(t)\,d\bar{\rvw},
\label{eq:reverse-process-approx}
\end{equation}
where $\rvx_T \sim \pi(\rvx_T)$. One can learn this score function by minimizing \cite{vincent2011connection}
\begin{align}
    \meanp{t,\rvx_0,\rvx_t}{\lambda(t) \norm{\score_\theta(\rvx_t; t) - \nabla_x\log q(\rvx_t | \rvx_0)}^2},
    \label{eq:objective-dm}
\end{align}
where $\rvx_0\sim q_0$ and $\rvx_t\sim q_t(\cdot | \rvx_0)$ in the expectation and $\lambda : [0, T] \rightarrow \R^+$ is a weighting function.
Once trained, one can generate data from the model by sampling $\rvx_T \sim \pi$ and simulating the approximated reverse process in \cref{eq:reverse-process-approx}.

In the remainder of this paper we use $q$ and $p_\theta$ respectively to denote the probability density function of the forward and reverse process. $P_\theta$ and $Q$ denote the probability mass functions associated with $p$ and $q$. This applies to the marginals, conditionals, and posteriors as well.
Furthermore, to reduce notational clutter throughout the rest of the paper, we omit the explicit mention of $\theta$ and $t$ when their meaning is evident from the context.

\subsection{Constrained Generative Modeling}
Constrained generative modeling tackles the problem of learning and generating from a distribution within a bounded domain $\Omega\subset \R^d$. This constrained domain $\Omega$ is either described by explicit constraints, e.g., a set of linear inequalities \cite{lou2023reflected}, or implicitly via binary functions taking $\rvx \in \R^d$ as inputs and indicating whether the constraints are satisfied \cite{naderiparizi2024dont}.
The training data in such problem is guaranteed to satisfy the given constraints. Formally, the dataset $\gD = \{\rvx_0^i\}_{i=1}^N$ follows a data distribution $q_0$ such that $Q_0(\Omega) = 1$.
The goal of constrained generative modeling is to approximate $q_0$ while being bounded to $\Omega$. A maximum likelihood estimation objective for this problem is formulated as
\begin{equation}
    \argmin_\theta \kl{q_0}{p_\theta} \quad \text{s.t. } P_\theta(\rvx \in \Omega) = 1.
\end{equation}

\subsection{Diffusion Bridges}\label{sec:background:db}
Diffusion bridges for constrained generative modeling was introduced by \citet{wu2022diffusion,liu2023learning}. It is a generalized framework of Brownian bridge processes and requires the constraint boundaries to be explicitly stated. For a constraint set $\Omega \subset \R^d$, a function $\gB^\Omega(\rvx_t, t)$ defined on $\R^d \times \R^+$ is an $\Omega$-bridge for the reverse process in \cref{eq:reverse-process-approx} if the solutions of
\begin{equation}
    d \rvx_t = [\nu_\theta(\rvx_t; t) - g^2(t) \gB^\Omega(\rvx_t; t)] dt + g(t) d\bar{\rvw},
    \label{eq:reverse-bridged-process}
\end{equation}
at final time $t=0$ are guaranteed to be in $\Omega$. Here $\rvx_T \sim \pi$ and $\nu_\theta(\rvx_t; t) := f(\rvx_t; t) - g(t)^2 s_\theta(\rvx_t; t)$.
Intuitively, the extra injected term $\gB^\Omega(\rvx_t; t)$ is quantified through external constraint functions, which drifts the particle to move towards the boundary and stay inside $\Omega$.
\citet{wu2022diffusion} provides a set of sufficient conditions for \cref{eq:reverse-bridged-process} to constitute a valid diffusion bridge.
However, these requirements on $\gB^\Omega$ make them practically applicable only to a very limited set of problems. For completeness, we provide these requirements in the appendix.
\citet{liu2023learning} proposes a particular diffusion bridge and shows it satisfies the necessary conditions. It, however, requires closed form access an expectation that is only tractable for very simple constraints such as product of intervals. We discuss this more in \cref{sec:method:manual-bridge}. For completeness, we show the exact form of this bridge in the appendix.
Furthermore, we show it corresponds to an optimal diffusion model trained on $\gU(\Omega)$, a uniform distribution on $\Omega$.

\section{Methodology}
In this section, we explain ``Manually Bridged Models'' (MBM), our approach to constrained generative modeling with manual bridges.
\method{} incorporates the complex constraint information into the model leading to a formulation similar to diffusion bridges. We show that such bridged models parameterize a family of sequence of distributions that only place mass on $\Omega$ at diffusion time $t=0$. The models are then trained using the same objective as standard diffusion models.

\subsection{Manual Bridges}\label{sec:method:manual-bridge}
Here, we first formally define the notion of manual bridges. Next, we explain how they are incorporated in diffusion models. Finally, we show how to combine multiple bridges to get a model that satisfies a set of given constraints.

\begin{definition}[$\Omega$-distance function]
    \label{def:omega-distance-fn}
    Let $\distancefn^\Omega: \R^d \times [0, T] \to \R^{\geq 0}$ be a continuous and almost everywhere differentiable function with finite gradients w.r.t. $x$. We call $\distancefn^\Omega$ an $\Omega$-distance function when $\distancefn^\Omega(\rvx; 0) = 0$ if and only if $\rvx \in \Omega$.
\end{definition}
\begin{definition}[Manually bridged model]
    \label{def:manually-bridged-model}
    Given a diffusion model $\score_\theta(\rvx_t, t)$, an $\Omega$-distance function $\distancefn^\Omega(\rvx; t)$, and a $C^1$-function $\gamma: [0, T] \to \R^+$ such that $\gamma(T) \approx 0$ and $\lim_{t \downarrow 0} \gamma(t) = \infty$, a \textbf{manual bridge} is defined as $\bridge^\Omega(\rvx; t) := -\gamma(t) \nabla_x \distancefn^\Omega(\rvx; t)$. A \textbf{manually bridged model} is defined as
    \begin{equation}
        \label{eq:manual-bridge-family}
        \score^\Omega_\theta(\rvx; t, \gamma, \distancefn) := \score_\theta(\rvx; t) + \bridge^\Omega(\rvx; t).
    \end{equation}
\end{definition}
Intuitively, the added manual bridge term guides the distribution towards $\Omega$. Manually bridged models correspond to score functions of distributions of the form $p^\Omega(\rvx; t) \propto p(\rvx; t) \exp(-\gamma(t) \distancefn^\Omega(\rvx; t))$. Since $\gamma$ smoothly changes from zero at $t=T$ to infinity at $t=0$, $p^\Omega(\rvx; t)$ smoothly interpolates between $p^\Omega(\rvx; T) = p(\rvx; T)$ at $t=T$ and $p^\Omega(\rvx; 0) \propto p(\rvx; 0) \1_\Omega(\rvx)$.
\begin{proposition}
    \label{prop:mbm-sequence-of-distributions}
    Let $s_\theta(\rvx, t)$ be a score function corresponding to a density $p_\theta(\rvx, t)$. If $s_\theta(\rvx, t)$ is continuous in $t$ and $p_\theta(\rvx, t)$ is finite for $\rvx \notin \Omega$, then the manually bridged model in \cref{def:manually-bridged-model} results in a sequence of distributions that only place mass on $\Omega$ at time $t=0$.
\end{proposition}
Proof of this proposition is provided in the appendix.

\begin{figure*}[t]
    \centering
    \begin{subfigure}{0.15\textwidth}
        \centering
        \captionsetup{justification=centering}
        \includegraphics[scale=1]{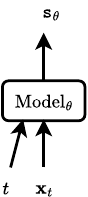}
        \caption{Standard diffusion}
        \label{fig:implementation-diagrams:standard}
    \end{subfigure}\hfill
    \begin{subfigure}{0.22\textwidth}
        \centering
        \captionsetup{justification=centering}
        \includegraphics[scale=1]{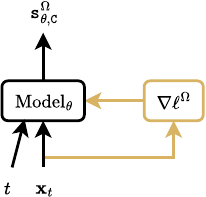}
        \caption{Conditional diffusion\\(\texttt{C-arch})}
        \label{fig:implementation-diagrams:c}
    \end{subfigure}\hfill
    \begin{subfigure}{0.3\textwidth}
        \centering
        \captionsetup{justification=centering}
        \includegraphics[scale=1]{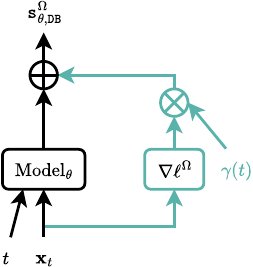}
        \caption{Diffusion Bridge architecture\\(\texttt{DB-arch})}
        \label{fig:implementation-diagrams:db}
    \end{subfigure}\hfill
    \begin{subfigure}{0.3\textwidth}
        \centering
        \captionsetup{justification=centering}
        \includegraphics[scale=1]{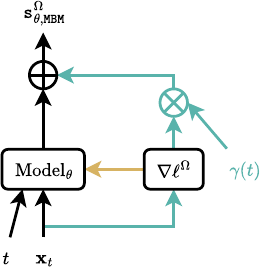}
        \caption{Our architecture\\(\texttt{MBM-arch})}
        \label{fig:implementation-diagrams:mbm}
    \end{subfigure}
    \caption{Score function architectural variants considered.  The latter three use the same ``manual bridge,'' with the last notably including an additional path for the bridge function gradient not previously considered in the literature. 
 Each diagram shows a single denoising step. In these diagrams the input $t$ to the model is omitted for conciseness.}
    \label{fig:implementation-diagrams}
\end{figure*}

\begin{proposition}[Combining Manual Bridges]
\label{prop:bridge-combination}
Let $\bridge^{\Omega_1}(\rvx; t) = -\gamma_1(t) \nabla_\rvx \distancefn^{\Omega_1}(\rvx; t)$ and $\bridge^{\Omega_2}(\rvx; t) = -\gamma_2(t) \nabla_\rvx \distancefn^{\Omega_2}(\rvx; t)$ be two manual bridges as defined in \cref{def:manually-bridged-model}. If $\overline{\Omega} := \Omega_1 \cap \Omega_2 \neq \emptyset$, the combined bridge $\bridge^{\overline{\Omega}}(\rvx; t) = \bridge^{\Omega_1}(\rvx; t) + \bridge^{\Omega_2}(\rvx; t)$ is also a manual bridge to $\Omega_1 \cap \Omega_2$. Therefore, the space of manual bridges is closed under addition.
\end{proposition}
\begin{proof}
    Without loss of generality, assume $\lim_{t \downarrow 0} \frac{\gamma_2(t)}{\gamma_1(t)} \neq 0$. Let $\distancefn^{\overline{\Omega}}(\rvx; t) := \distancefn^{\Omega_1}(\rvx; t) + \frac{\gamma_2(t)}{\gamma_1(t)} \distancefn^{\Omega_2}(\rvx; t)$.
    Since all functions are continuous in $t$, both $\distancefn^{\Omega_1}$ and $\distancefn^{\Omega_2}$ are distance functions and $\lim_{t \downarrow 0} \frac{\gamma_2(t)}{\gamma_1(t)} \neq 0$, $\distancefn^{\overline{\Omega}}(\rvx; 0)$ is zero if and only if both $\distancefn^{\Omega_1}(\rvx; 0)$ and $\distancefn^{\Omega_2}(\rvx; 0)$ are zero. Therefore, $\distancefn^{\overline{\Omega}}$ is a distance function to $\overline{\Omega}$.
    Further, by definition $\lim_{t \downarrow 0} \gamma_1(t) = \infty$.
    Therefore, $-\gamma_1(t) \nabla_x \distancefn^{\overline{\Omega}}(\rvx; t) = \bridge^{\Omega_1}(\rvx; t) + \bridge^{\Omega_2}(\rvx; t) = \bridge^{\overline{\Omega}}(\rvx; t)$ is a manual bridge to $\overline{\Omega}$.
\end{proof}
Generalizing Proposition 1, one can combine multiple bridges by $\bridge^{\cap_{i=1}^N\{\Omega_i\}} = \sum_{i=1}^{N} {\bridge^{\Omega}_i}(\rvx; t)$.

The objective function of \method{} is the denoising loss with this particular parameterization of the score estimator model
\begin{equation}
    \meanp{t, \rvx_0, \rvx_t}{\lambda(t) \norm{\score_\theta(\rvx_t; t, \gamma, \distancefn) - 
    \nabla_{x}\log q_t(\rvx_t | \rvx_0)}^2}
\end{equation}
where $\rvx_0\sim q_0^\Omega, \rvx_t\sim q_t(\cdot | \rvx_0)$ in the expectation.

\paragraph{Connection to diffusion bridges} 
Note that our manually bridged models give rise to reverse SDEs similar to that of diffusion bridges in \cref{eq:reverse-bridged-process}. To see this, we can plug the manually bridged score function in \cref{eq:reverse-process-approx}:
\begin{align}
    d\rvx_t &= [\nu_\theta(\rvx_t; t) 
    - g^2(t) \bridge^\Omega(\rvx_t; t)]\,dt + g(t)\,d\bar{\rvw},
\end{align}
where $\nu_\theta(\rvx_t; t) \equiv f(\rvx_t; t) - g^2(t)\score_\theta(\rvx_t; t)$.
This is equivalent to \cref{eq:reverse-bridged-process} once $\gB^\Omega(\rvx_t; t) \equiv \bridge^\Omega(\rvx_t; t)$. However, in order to guarantee that solutions of this SDE lie in $\Omega$, the bridge function $\bridge^\Omega(\rvx, t) = -\gamma(t) \nabla_\rvx\distancefn^\Omega(\rvx; t)$ must satisfy requirements such as an expected Polyak-Lojasiewicz condition: $\meanp{t, \rvx_t \sim p_t^\Omega}{\distancefn^\Omega(\rvx_t; t)} \leq \meanp{t, \rvx_t \sim p_t^\Omega}{\norm{\nabla_x \distancefn^\Omega(\rvx_t; t)}^2}$. This greatly restricts the set of allowed $\Omega$-distance functions. Furthermore, it restricts generality of combination of bridges.
Our manually bridged models therefore are not guaranteed to have SDE solutions in $\Omega$. However, as shown in the previous section, they still represent score functions of distributions that converge to one constrained to $\Omega$. Further, since we train the model to approximate the reverse process in \cref{eq:reverse-process}, an SDE with a trained model is likely to converge to this terminal distribution as well. As we show later in the Experiment section, manual bridges empirically perform similarly to diffusion bridges.

\begin{figure*}[t]
    \centering
    \begin{subfigure}{0.24\textwidth}
        \captionsetup{justification=centering}
        \includegraphics[width=4cm]{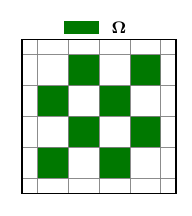}
        \caption{Constraint set}
        \label{fig:toy-vis:constraint}
    \end{subfigure}\hfill
    \begin{subfigure}{0.24\textwidth}
        \captionsetup{justification=centering}
        \includegraphics[width=4cm]{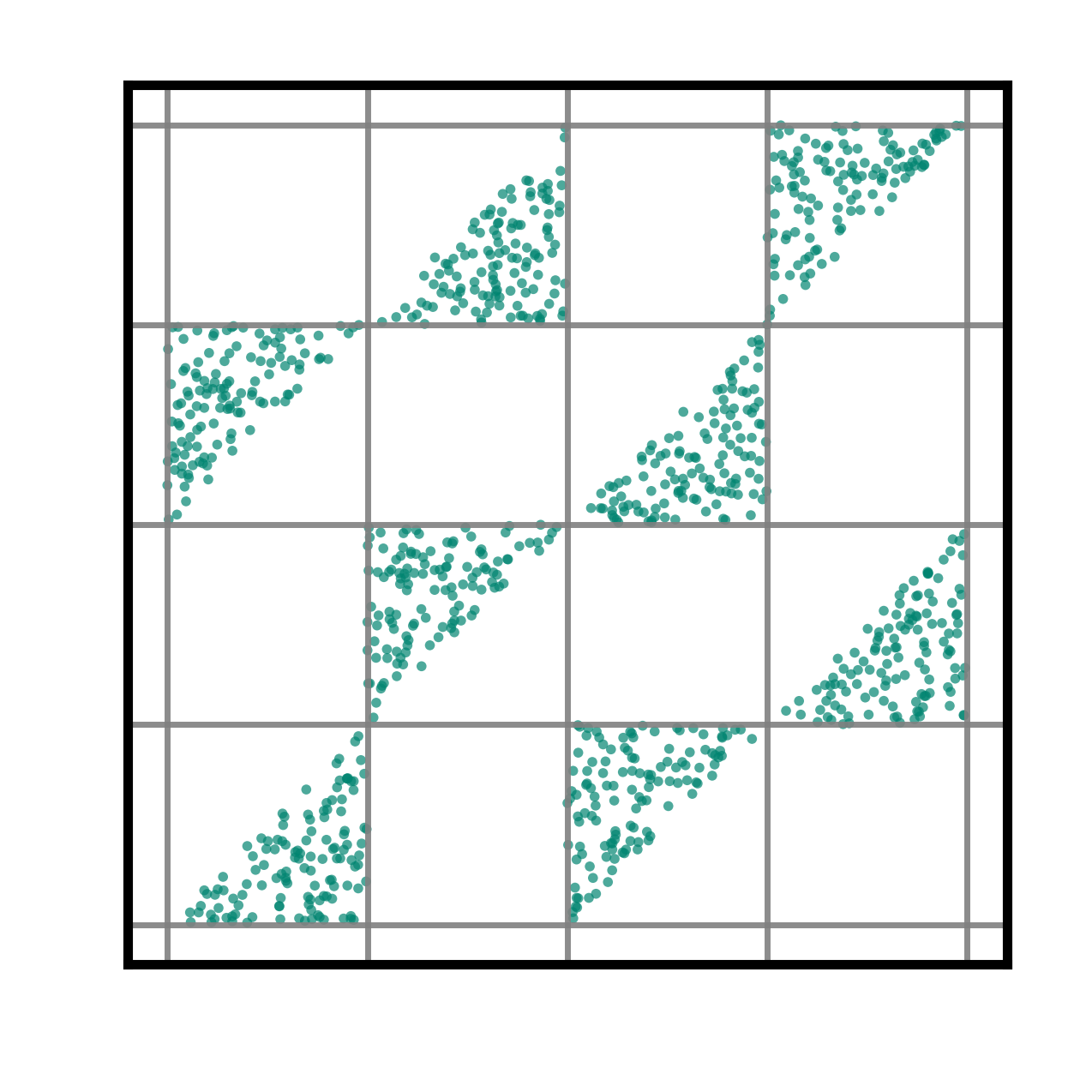}
        \caption{Data distribution}
        \label{fig:toy-vis:data}
    \end{subfigure}\hfill
    \begin{subfigure}{0.24\textwidth}
        \captionsetup{justification=centering}
        \includegraphics[width=4cm]{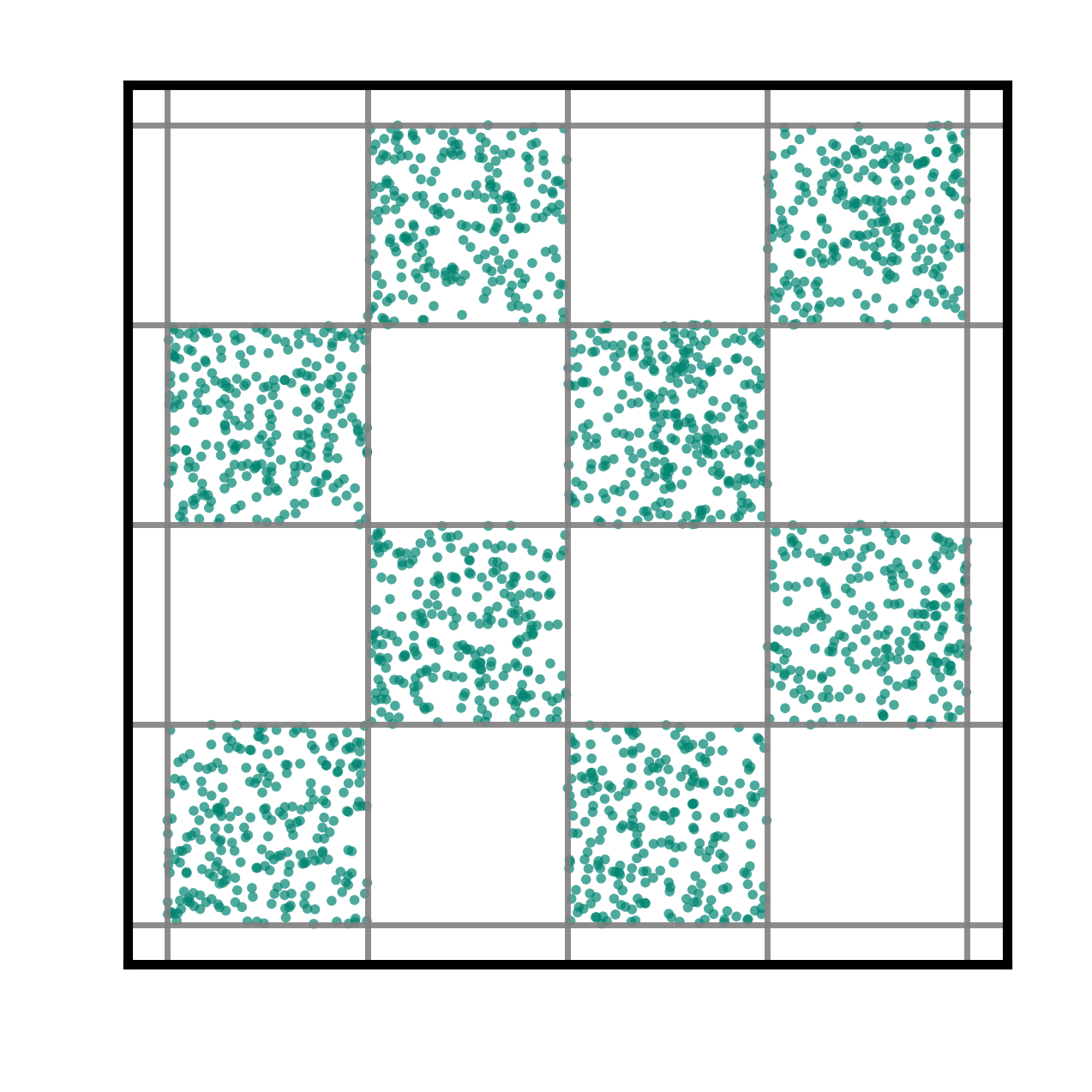}
        \caption{Prior diff. bridge}
        \label{fig:toy-vis:prior-analytic}
    \end{subfigure}\hfill
    \begin{subfigure}{0.24\textwidth}
        \captionsetup{justification=centering}
        \centering
        \includegraphics[width=4cm]{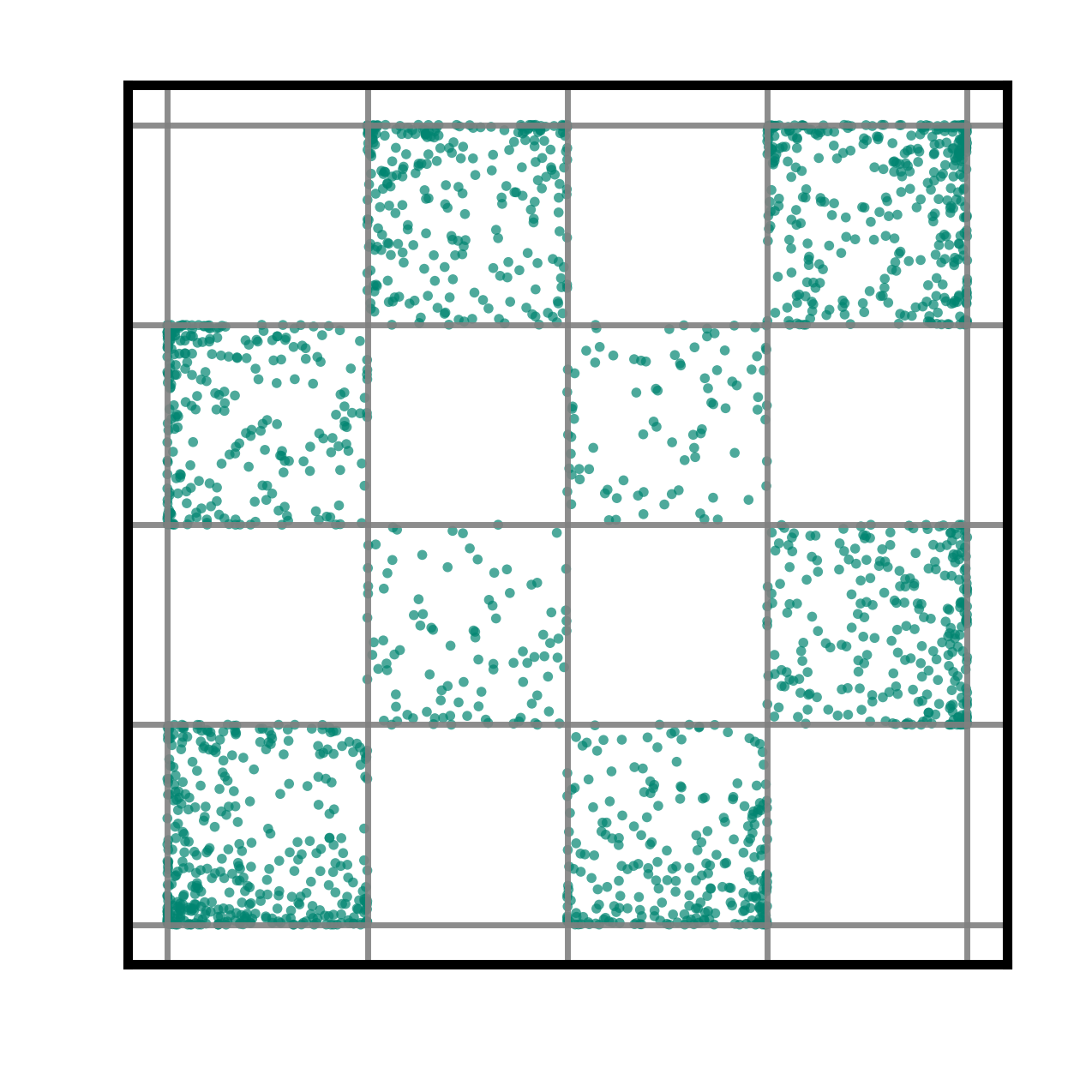}
        \caption{Prior manual bridge}
        \label{fig:toy-vis:prior-manual}
    \end{subfigure}
    \centering
    \begin{subfigure}{0.33\textwidth}
        \centering
        \captionsetup{justification=centering}
        \includegraphics[width=4cm]{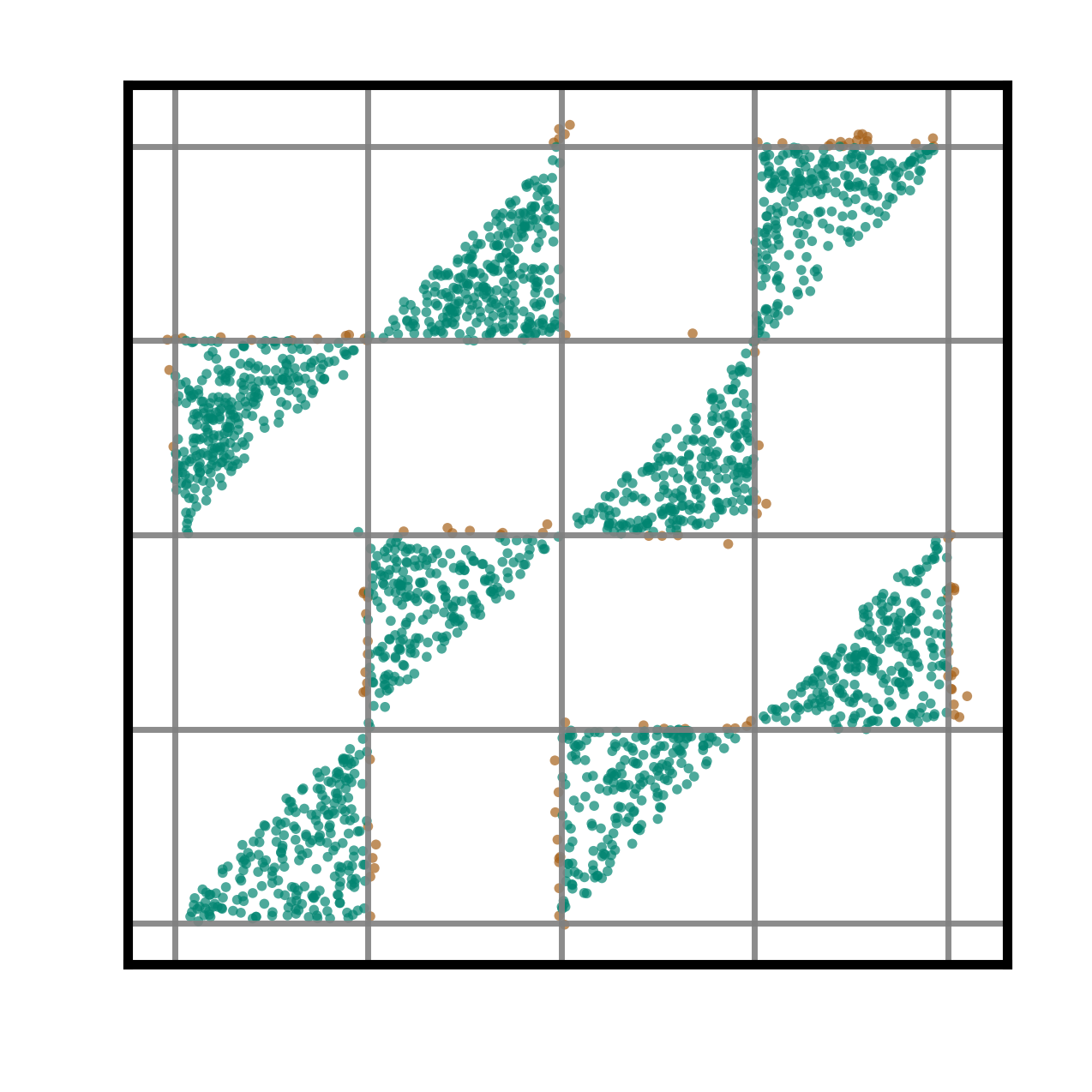}
        \caption{Baseline\\$\,$}
        \label{fig:toy-vis:baseline}
    \end{subfigure}\hfill
    \begin{subfigure}{0.33\textwidth}
        \captionsetup{justification=centering}
        \centering
        \includegraphics[width=4cm]{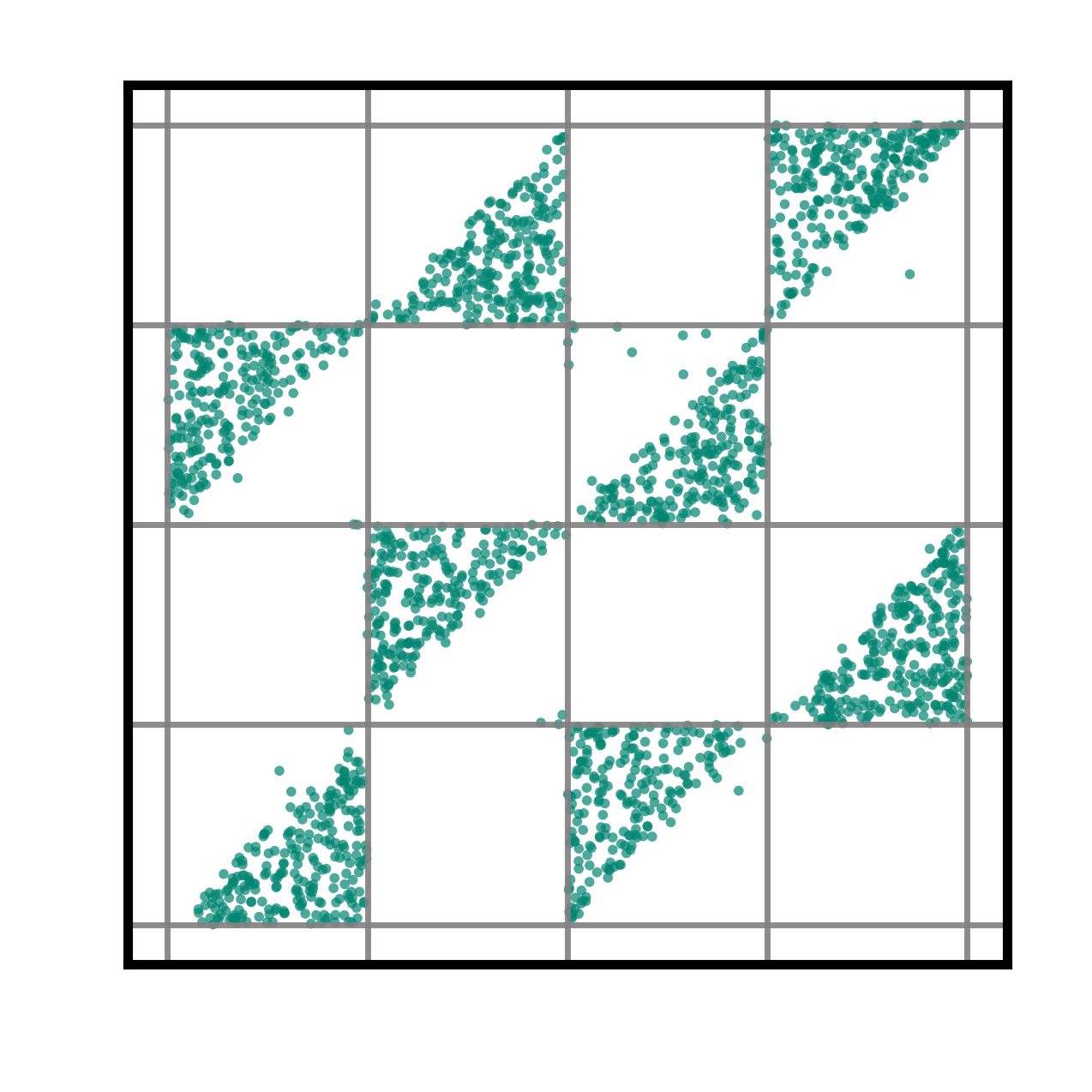}
        \caption{Diffusion bridge~\citep{liu2023learning}\\with \texttt{DB-arch}}
        \label{fig:toy-vis:analytic-sde}
    \end{subfigure}\hfill
    \begin{subfigure}{0.33\textwidth}
        \captionsetup{justification=centering}
        \centering
        \includegraphics[width=4cm]{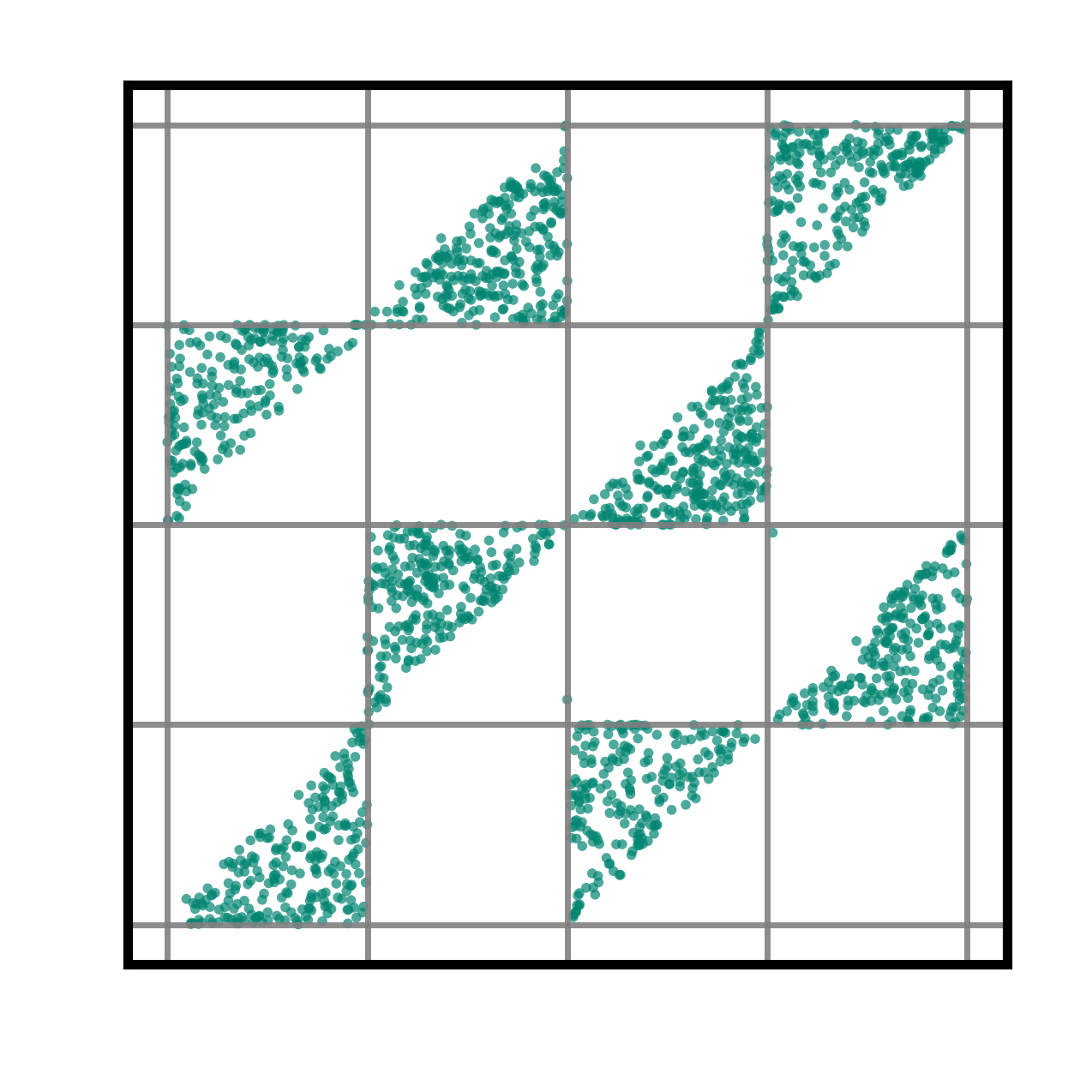}
        \caption{Manual bridge\\with \texttt{MBM-arch} (ours)}
        \label{fig:toy-vis:manual-sde-inp}
    \end{subfigure}\hfill
    \label{fig:toy-vis}
    \caption{Visualization of the checkerboard constraint experiment results. The problem is constrained on a checkerboard pattern and the data has a uniform distribution over triangles within the checkerboard, shown in (\subref{fig:toy-vis:data}). Invalid samples are shown in brown. (\subref{fig:toy-vis:prior-analytic}, \subref{fig:toy-vis:prior-manual}) show the diffusion bridge and manually bridged models without a trained diffusion model. This is effectively the prior distributions in these models. As shown on the second row, bridged models do not produce invalid samples. Further, the manually bridged model gives comparable samples to the diffusion bridges \cite{liu2023learning}. Finally, incorporating the bridges in our proposed way (labelled with \method{}) even improves the diffusion bridge models.}
\end{figure*}

\subsection{Architecture}
\label{sec:method:implementation}
In practice, the typical diffusion bridge method of incorporating manual bridges as in \cref{eq:manual-bridge-family} leads to poor training. This is because the added bridge term increases variance of the loss.
As illustrated in \cref{fig:implementation-diagrams}, we consider three different types of incorporating the bridge information into the model.
\begin{enumerate}
    \item \textbf{Conditional diffusion (\texttt{C-arch})}:  modifies the score network to take (a re-weighted version of) the bridge as an additional input. This effectively becomes a conditional diffusion model,  and those have been proven to be highly effective in, e.g., large-scale text-conditional image generation tasks. Note that this mechanism merely incorporates the information from the distance function to the score network. As such, it is not a bridged model and does not provide guarantees regarding the constraints. However, its training is stable.
    \begin{equation}
        \score_{\theta, \texttt{C}}^\Omega(\rvx_t; t, \gamma, \distancefn) := \score_\theta(\rvx_t; t, \nabla_x \distancefn^\Omega(\rvx_t))
        \label{eq:implementation-conditional-dm}
    \end{equation}
    \item \textbf{Diffusion Bridges (\texttt{DB-arch})}:  offsets the score function using the bridge function. This is the mechanism used in diffusion bridge \cite{wu2022diffusion, liu2023learning}.  This mechanism is effective at constraining the distribution but results in unstable training.  
    \begin{equation}
        \score^\Omega_{\theta, \texttt{DB}}(\rvx_t; t, \gamma, \distancefn) := \score_\theta(\rvx_t; t) + \bridge^\Omega(\rvx_t; t)
        \label{eq:implementation-diffusion-bridges}
    \end{equation}
    \item \textbf{Manually Bridged Models (\texttt{MBM-arch}; ours)}:
    combines the above two mechanisms.  We have found that additionally providing (a re-weighted version of) the bridge to the score network stablizes training resulting in both a good model fit and good constraint satisfaction.
    \begin{equation}
        \score_{\theta, \texttt{MBM}}^\Omega(\rvx_t; t, \gamma, \distancefn) := \score_\theta(\rvx_t; t, \nabla_x \distancefn^\Omega(\rvx_t)) + \bridge^\Omega(\rvx_t; t)
        \label{eq:implementation-ours}
    \end{equation}
\end{enumerate}

While any of the three mechanisms above can be used to incorporate either manual or diffusion bridges, our proposed method, referred to as MBM, involves \textit{manual bridges} applied through the \texttt{MBM-arch} mechanism.

\section{Experiments}
We demonstrate \method{} on a simple 2D synthetic dataset and a traffic scenario generation experiment with collision and offroad avoidance.
Additionally, we include an image watermarking experiment in the appendix.
We release the source code implementing \method{} together with the 2D synthetic and image watermarking experiments.

\begin{figure*}
    \centering
    \includegraphics[scale=1]{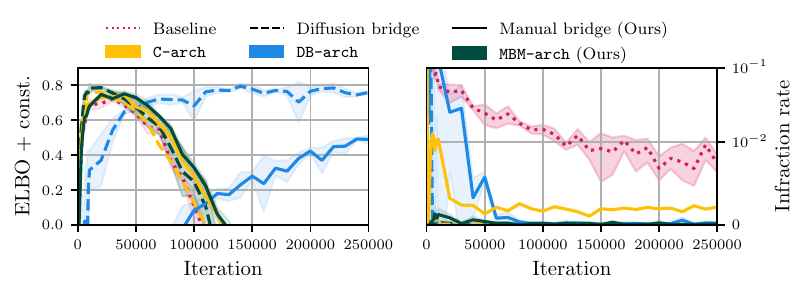}
    \caption{Checkerboard constraint experiment results. Solid, dashed and dotted lines respectively correspond to manual bridge, diffusion bridge and baseline models. Different colors represent different mechanisms of incorporating bridges as shown in \cref{fig:implementation-diagrams}. The shaded area shows standard deviation of the metrics over three models trained with different random seeds.}
    \label{fig:toy-metrics}
\end{figure*}

\subsection{Checkerboard Constraint Experiment}
\paragraph{Setup} We first experiment with a simple 2D dataset with a checkerboard constraint (\cref{fig:toy-vis:constraint}). The data distribution $q_0$ is a mixture of uniform triangles enclosed in the allowed checkerboard area. Our dataset consists of 1,000 samples from this data distribution (\cref{fig:toy-vis:data}). We report results for a baseline diffusion model in \cref{fig:toy-vis:baseline}. This is the standard diffusion as explained in Background section. Samples from this baseline model include quite a few constraint violations.

\paragraph{Diffusion bridge baseline} As stated before, diffusion bridges \citep{liu2023learning} are only tractable on very simple constraints. The checkerboard pattern of this experiment is one example.
As stated before and shown in the appendix, the diffusion bridge for this problem is equivalent an exact diffusion model for uniformly distributed data in the constraint set $\Omega$.
We empirically verify this in \cref{fig:toy-vis:prior-analytic}, where we construct an SDE with only the bridge.
This is equivalent to \cref{eq:reverse-bridged-process} without the $\nu$ term.
As such, the results in Fig. 3c match $\gU(\Omega$).
Consequently, we modify the typical checkerboard data distribution to be non-uniform in the constraint set.  The ``squares'' in \cref{fig:toy-vis:constraint} are constraints.  The data just happens to lie in half of the constraint set.  This modification is necessary to highlight model training within the diffusion bridge framework -- the bridge to the uniform checkerboard is literally the generative model for the typical checkerboard dataset.

Following \citet{liu2023learning}, we implement a model with \texttt{DB-arch} (see \cref{fig:implementation-diagrams:db}).
Once this bridged model is trained, it will match the data distribution while staying within the constraints (\cref{fig:toy-vis:analytic-sde}).

\paragraph{Manual bridges} For manual bridges we use $\distancefn^\Omega(\rvx; t) := \min_{\rvy \in \Omega} \norm{\rvx - \rvy}_2^2$ as the distance function and $\gamma(t) := \frac{1}{\sigma^2(t)}$ where $\sigma(t)$ is the total amount of noise added at time $t$ of the diffusion process.
We demonstrate in \cref{fig:toy-vis:prior-manual} that our manually bridged model guides the generation differently from diffusion models. Similar to \cref{fig:toy-vis:prior-analytic}, this is resulted from an SDE with only the bridge. \cref{fig:toy-vis:manual-sde-inp} shows that a Manually Bridged Model implemented in our proposed mechanism \texttt{MBM-arch} (\cref{fig:implementation-diagrams:mbm})
produces comparable results to a diffusion bridge model.

\paragraph{Quantitative results} We report our quantitative results in \cref{fig:toy-metrics}. We report evidence lower bound (ELBO) and infraction rate to respectively measure distribution match and constraint satisfaction rate. We compare standard diffusion model with both types of bridged models (diffusion and manual bridges) as well as different mechanisms for incorporating bridges.

Although passing the bridge information to the model to condition on (as done in conditional diffusion models, labeled \texttt{C-arch} in the plot) reduces infraction rate, it fails to achieve zero infraction rate. With bridges incorporated, both diffusion bridges or manual bridges brings the infraction rate down to (almost) zero with both \texttt{DB-arch} and \texttt{MBM-arch} mechanism. However, the ELBO plot shows that the \texttt{DB-arch} models are much slower to converge. It takes the \texttt{DB-arch} models around 250k iterations to achieve their maximum validation ELBO while the other models achieve the maximum at around 20k-30k iterations and start over-fitting after.
It demonstrates that our proposed \texttt{MBM-arch} mechanism or incorporating bridges is even effective for implementing diffusion bridges of \citet{liu2023learning}.
This plot also verifies that while diffusion bridges have a more solid grounding with guaranteed convergence to an $\Omega$-constrained distribution, our proposed manual bridges can successfully achieve similar results. This makes manual bridges a much simpler and more widely applicable alternatives to diffusion bridges.

\begin{table*}[t] 
  \centering
  \begin{tabular}{lllll}
    \toprule
    \centering
    Method     & Collision ($\%$)  $\downarrow$  & Offroad ($\%$) $\downarrow$ & Infraction ($\%$) $\downarrow$ & r-ELBO  $\uparrow$  \\
    \midrule
    Standard diffusion~\cite{karras2022elucidating} & $17.40\pm 0.01 $  & $7.33\pm 0.18$ & $23.00\pm 0.14$ & $-0.94\pm 0.01$ \\
    Guided diffusion & $0.00\pm 0.00$ & $0.00\pm 0.00$ & $0.00\pm 0.00$ & $-1.53\pm 0.11$  \\
    \midrule
    Manual bridge with $\texttt{C-arch}$ & $19.47\pm 0.20$ & $6.87\pm 0.03$ & $24.52\pm 0.23$ & $-0.94\pm 0.00$ \\
    Manual bridge with $\texttt{DB-arch}$ & $0.24\pm 0.03$ & $0.00\pm 0.00$ & $0.25 \pm 0.03$ & $-1.20\pm 0.01$  \\
    Manual bridge with $\texttt{MBM-arch}$ (ours)    & $0.10\pm 0.00$ &  $0.00\pm 0.00$ & $0.10\pm 0.00$ & $-0.95\pm 0.00$ \\
    \bottomrule
  \end{tabular}
  \caption{Results for traffic scene generation.
  We compare our model $\texttt{MBM-arch}$ against its alternative architectures $\texttt{C-arch}$ and $\texttt{DB-arch}$, a standard diffusion model, and a guided variant using the bridge term as a guidance signal.
  }
  \label{tab:ic-experiment}
\end{table*}

\subsection{Traffic Scene Generation}\label{sec:exp:ic}
We continue to the experiment of generating realistic traffic scenes for placing various number of arbitrary-sized vehicles on different maps. Since a major part of traffic simulation tasks focuses on learning driving behaviors and predicting vehicle trajectories, the task of generating a realistic initial traffic scene before the traffic simulations is equally significant. A common practice to determine whether a generated traffic scene sample is valid, is to check if any vehicle is outside the drivable region (called ``offroad''), or if vehicles overlap one another (called ``collision''). To deal with the invalid samples being generated, the previous methods for this problem usually simply discard them~\cite{tan2021scenegen, zwartsenberg2022conditional}. This requires repeated sampling from the model and is computationally inefficient. A more recent method updates the model post-hoc by guiding it via a separately trained model to lower the infraction rate~\cite{naderiparizi2024dont}. While it lowers the infraction rate, it is still far from the ideal of zero infraction and requires extra computing resources. \method{} is constructed to achieve (almost) zero infraction rate, hence eliminating the requirement of repeated sampling or post-hoc modifications.

The problem in this experiment is to generate up to 25 vehicles on a bird's-eye view image of a road from one of 70 locations in the dataset. To condition on the road, we extract its image features using a convolutional neural network (CNN) encoder \citep{lecun1998gradient} and pass them to the score network. The CNN is trained jointly with the score network. Each vehicle is represented by its center position, length, width, heading direction and velocity, which makes 7 dimensions per vehicle. Note that vehicles are considered jointly, and their interactions are modeled in full, so that the overall dimensionality of the problem is $N \times 7$, with $N$ the total number of vehicles. As alluded to earlier, there are two constraints in this experiment: collision and offroad. For the collision constraint we compute the area of intersection between two vehicles as the $\Omega_c$-distance function. For offroad, we consider a car that has all four wheels off the drivable area as offroad. Therefore, for the $\Omega_o$-distance, we compute the shortest squared distance to the constraint set of each of the four corners of the vehicle and take the minimum. The $\gamma$ functions for collision and offroad are respectively $\gamma_c(t) = \frac{1}{10\sigma^2(t)}$ and $\gamma_o(t) = \frac{1}{100\sigma^2(t)}$. The complex constraints of this problem renders diffusion bridges \citep{liu2023learning} and most of the existing diffusion-based constrained generative modeling approaches \citep{lou2023reflected,fishman2023diffusion,liu2024mirror,fishman2024metropolis} inapplicable.

We implement a diffusion model based on EDM~\cite{karras2022elucidating}. Our model has a transformer-based architecture~\cite{vaswani2017attention} composed of self-attention and cross-attention layers modeling the interactions between vehicles and between vehicles and the road map.
For evaluation, we report collision, offroad, and overall infraction rate to measure constraint satisfaction and the validation loss which corresponds to a reweighted ELBO (r-ELBO) to evaluate the distribution match.
The standard diffusion serves as the baseline.
We also include a guided diffusion variant that guides this pre-trained diffusion with our manual bridges without further training.
We compare these with different mechanisms for incorporating manual bridges as described in~\cref{sec:method:implementation}.
\cref{tab:ic-experiment} reports our results for this experiment.
Adding the manual bridge as an offset to the model, $\texttt{DB-arch}$, strongly reduces infractions but it deteriorates model quality significantly as evidenced by the degraded r-ELBO. While guided diffusion achieves zero infraction, its significantly worse r-ELBO suggests a strong distribution shift.
$\texttt{MBM-arch}$ ultimately produces a model with close to zero infractions, and importantly recovers performance in resulting r-ELBO.
In summary, \method{} with our proposed \texttt{MBM-arch} mechanism is able to effectively nearly solve the problem of infraction free traffic scene generation.

\section{Discussion}
We introduced manually bridged models, a family of score functions corresponding to distributions that anneal between a known base distribution $\pi(\rvx_T)$ to an $\Omega$-constrained distribution at time $t=0$. Manual bridges can, for instance, be used to impose prior knowledge in the form of safety constraints. We showed
\begin{itemize}
    \item one can construct such manual bridges having access to a distance function to the constraint set,
    \item such bridges can be combined to get a multiply-constrained model,
    \item learning diffusion processes on top of manual bridges can fit observational data while maintaining safety constraints and overcoming potentially unwelcome biases imposed by suboptimal manual bridges,
    \item architecture matters when imposing bridges, particularly for effective learning. We propose a combination of (i) including the manual bridge in the score function in form of addition and (ii) conditioning the score network on the bridge. We empirically demonstrate that this combination is crucial in training of bridged models.
\end{itemize}

\paragraph{Limitations}
While \method{} is much more flexible and widely applicable to constrained generative modeling problems, it still relies on differentiable distance functions, which may not be straightforward to implement for certain problems.
Further, it is quite likely that there is a greater family of manual bridges that can be described and composed than we have established.  This includes both the bridge primitives and the methods of combining bridge functions.

The proper functioning of manual bridges still requires a fair amount of hyperparameter fiddling, particularly integration schedule, minimum noise level, and bridge scaling.  This is related to the gap between numerical and exact integration, however, it bears mentioning here as the asymptotic scaling of bridges and their gradients is extreme; sufficiently so to be fiddly.  The specific inspiration for the architecture innovation we introduced came from this low noise asymptotic scaling difficulty.  Even with our architectural innovation and manual bridge conditions, applying this method to new domains is not automatic.  In this sense it is not terribly different to prior specification in Bayesian models, except here the error mode looks more like getting garbage samples that are ``safe.''

\paragraph{Future Work}
There is a great deal of theory work to be done, particularly with respect to how manual bridges fit into the diffusion bridge formalism.  We would ideally like to  identify the mathematical conditions on manual bridges and their combinations that would give rise to formal guarantees of constraint satisfaction.  Thus far, it appears that we need a generalization of Ito's lemma that relaxes the primary twice differentiable condition to twice differentiable almost everywhere or even looser.  

While solving the constrained, conditional agent placement distribution problem is of value; it is obvious that our approach should and could be extended to trajectory space as well.  Static actor placements are, after all, merely short trajectories.  Our initial attempts at this were promising, but, work remains to be done.

\section*{Acknowledgments}
We acknowledge the support of the Natural Sciences and Engineering Research Council of Canada (NSERC), the Canada CIFAR AI Chairs Program, Inverted AI, MITACS, the Department of Energy through Lawrence Berkeley National Laboratory, and Google. This research was enabled in part by technical support and computational resources provided by the Digital Research Alliance of Canada Compute Canada (alliancecan.ca), the Advanced Research Computing at the University of British Columbia (arc.ubc.ca), and Amazon.

\bibliography{aaai25}
\clearpage
\input{appendix.tex}

\end{document}

%% file: appendix.tex
\onecolumn
\appendix

\section{Table of Notations}
\begin{table*}[h] 
\centering
\begin{tabular}{ll}
    \toprule
    \centering
    \textbf{Symbol}     & \textbf{Description}\\
    \midrule
    $q_0(\rvx_0)$ & Density of data distribution.\\
    $q_t(\rvx_t)$ & Marginal distribution density of forward process at time $t$.\\
    $q_t(\rvx_t | \rvx_0)$ & Conditional distribution density of forward process at time $t$ given a data point $\rvx_0$.\\
    $\nabla_x \log q(\rvx_t)$ & Score function of the marginal distribution of forward process at time $t$.\\
    $\theta$ & Parameters of a diffusion model (usually weights and biases of a neural network)\\
    $\score_\theta(\rvx_t; t)$ & Score function learned by the model.\\
    $p_\theta(\rvx_t)$ & Marginal distribution density of the reverse process defined by $\score_\theta(\rvx_t; t)$.\\
    $P_\theta(\rvx)$ & The probability distribution corresponding to the reverse process implied by $\score_\theta(\rvx_t; t)$, at $t=0$.\\
    $\Omega$ & Constraint set.\\
    $\distancefn^\Omega(\rvx_t; t)$ & A distance function to $\Omega$.\\
    $\gamma(t)$ & Scaling of distance function to define manual bridges.\\
    $\bridge(\rvx_t; t)$ & A manual bridge. It has the form of $\gamma(t) \nabla_x \distancefn^\Omega(\rvx_t; t)$.\\
    $\score_{\theta}^\Omega$ & Score function defined by a manually bridged model.\\
    $\score_{\theta, \texttt{C}}^\Omega$ & Score function defined by a \texttt{C-arch} with a manual bridge conditioning.\\
    $\score_{\theta, \texttt{DB}}^\Omega$ & Score function defined by a manually bridged model with \texttt{DB-arch} architecture.\\
    $\score_{\theta, \texttt{MBM}}^\Omega$ & Score function defined by a manually bridged model with \texttt{MBM-arch} architecture.\\
    $\gB^\Omega(\rvx_t; t)$ & A notation to generically refer to diffusion bridge updates.\\
    \bottomrule
\end{tabular}
\end{table*}

\section{Experimental Details}
\subsection{Checkerboard constraint experiment}
\paragraph{Architecture} Our model's architecture is a fully connected multi-layer perceptron (MLP) with 2 residual blocks similar to that of \citet{naderiparizi2024dont}. We use a hidden layer dimension of 256 and sinusoidal timestep dimension of 128.
\paragraph{Diffusion process} We use the EDM \citep{karras2022elucidating} framework with similar hyperparameters. However, since in manual bridges it is important to be accurate when $t$ is close to zero, we modify the distribution $p_\sigma$ of \citet{karras2022elucidating} to be a log-linear with $\sigma_{\text{min}} = 3 \times 10^{-5}$ and $\sigma_{\text{max}} = 80$.
\paragraph{Sampling} We use the second-order Heun solver of \citet{karras2022elucidating} with 100 sampling steps and $S_{\text{churn}} = 10$. We also change the schedule of $\sigma$ to a log-linear from $\sigma_{\text{max}}$ to $\sigma_{\text{min}}$, similar to the training distribution.
\paragraph{Training} We train our models with the Adam optimizer \citep{kingma2014adam} with a learning rate of $3 \times 10^{-4}$ and batch size of $1000$. Other hyperparameters used are the default values in PyTorch \citep{paszke2019pytorch}. All the models were trained on GeForce GTX 1080
Ti GPUs. The total estimated compute for this experiment is about 10 GPU hours.

\subsection{Traffic Scene Generation experiment}
\paragraph{Architecture} The architecture for training models is based on transformers, whose backbone is composed of an encoder, a stack of self-attention with relative positional encodings (RPEs)~\citep{shaw2018self, wu2021rethinking,harvey2022flexible} and cross-attention residual blocks and a decoder. We refer readers to~\citet{naderiparizi2024dont} for more details with no extra modification compared to this work. With the same setup, we train each model with a batch size of $64$ and Adam optimizer \citep{kingma2014adam} with a learning rate of $10^{-4}$. The architecture contain approximately $6.3$ million parameters. To embed bridges in $\texttt{C-arch}$ or $\texttt{MBM-arch}$, we process it with an additional two-layer MLP with activation function SiLU. We use Tesla V100 for training and validating models with around $200$ hrs GPU-hours in total.

\paragraph{Diffusion process} Similar to Toy experiment, we deploy EDM~\citep{karras2022elucidating} framework and set $\sigma_{\text{min}} = 2\times 10^{-4}$ with a log-linear distribution $p_\sigma$.
\paragraph{Sampling} We use the first-order Euler-Maruyama solver~\citep{song2020score} with 300 sampling steps, set $S_{\text{churn}} = 10$ and change $\sigma$ to the log-linear schedule as the Checkerboard constraint experiment.
\paragraph{Training} We first train Standard diffusion for $950$k iterations with batch size 64. \method{} with $\texttt{C-arch}$, $\texttt{DB-arch}$ and $\texttt{MBM-arch}$ respectively, are fine-tuned based on the Standard diffusion and trained them $300$k iterations longer.
\paragraph{Bridge computation} While incorporating bridge term $\bridge^\Omega(\rvx_t; t)$ into \method{}-style training process, an issue was observed: when the noise level $\sigma_t$ injected in training sample $\rvx_t$ is high, and both vehicle positions and sizes are diffused, abnormally-huge-sized vehicles are ejected from the map roads, resulting in both ``offroad'' and ``collision'' losses increasing significantly. The gradients do not provide informative guidance to the model but simply dominate the training over the score function, which destabilizes the training and is not ideal. To stabilize the training with bridges integrated, we normalized $\rvx_t$ by $\sqrt{\Var(\rvx_t)} = \sqrt{\Var(\rvx_0 + \sigma\epsilon)} = \sqrt{1 + \sigma_t^2}$ before passing it to $\ell^\Omega(\cdot)$. With chain rule applied, we have
\begin{align}
&\nabla_{\rvx_t} \ell^{\Omega}(\rvx_t) = \frac{\partial\ell^{\Omega}(\rvx_t')}{\partial \rvx_t'} \frac{\partial \rvx_t'}{\partial \rvx_t},\ \text{where}\ \rvx_t' = \frac{\rvx_t}{\sqrt{1 + \sigma_t^2}} \\
\Rightarrow &\nabla_{\rvx_t} \ell^{\Omega}(\rvx_t) = \frac{1}{\sqrt{1 + \sigma_t^2}}\frac{\partial\ell^{\Omega}(\rvx_t')}{\partial\rvx_t'}
\end{align}
which effectively alleviates the problem, as the loss function evaluates the normal-sized vehicles centered on the road map.

\begin{figure*}[tb]
    \centering
    \includegraphics{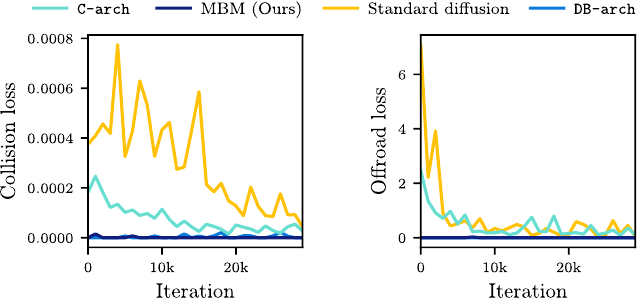}
    \caption{While standard diffusion models struggle to achieve zero infraction, incorporating diffusion bridges ensures constraint satisfaction.}
    \label{fig:app:additional:plots}
\end{figure*}

\section{Diffusion bridges}
\subsection{Sufficient conditions}\label{app:diffusion-bridges:sufficient}
In this section we rephrase the theory of diffusion bridges from \citet{wu2022diffusion} in our notation. Following the definitions in \cref{sec:method:manual-bridge}, let $\distancefn^\Omega(\rvx; t)$ be an $\Omega$-distance function and let $\gamma(t)$ be a weighting function. Define $\gB(\rvx; t) = - \gamma(t) \distancefn^\Omega(\rvx; t)$. The SDE in \cref{eq:reverse-bridged-process} is a bridge to $\Omega$ if the following holds
\begin{enumerate}
    \item $\distancefn^\Omega$ follows an (expected) Polyak-Lojasiewicz condition: $\meanp{t, \rvx_t \sim p_t^\Omega}{\distancefn^\Omega(\rvx_t; t)} \leq \meanp{t, \rvx_t \sim p_t^\Omega}{\norm{\nabla_x \distancefn^\Omega(\rvx_t; t)}^2}$ for all $t \in [0, T]$,
    \item Let
    \begin{itemize}
        \item $\beta(t) = \meanp{t, \rvx_t \sim p_t^\Omega}{\nabla_x \distancefn^\Omega(\rvx_t; t)^\transpose \nu(\rvx_t; t)}$,
        \item $\rho(t) = \E_{t, \rvx_t \sim p_t^\Omega} \Big[ \partial_t \distancefn^\Omega(\rvx_t; t) + \frac{1}{2} \Tr(\nabla^2_x \distancefn^\Omega(\rvx_t; t)g^2(t)) \Big]$,
        \item $\zeta(t) = \exp(\int_t^T \gamma(s) ds)$.
    \end{itemize}
    Then
    \begin{itemize}
        \item $\lim_{t \downarrow 0} \zeta(t) = +\infty,$
        \item $\lim_{t \downarrow 0} \frac{\zeta(t)}{\int_t^T \zeta(s) (\beta(s) + \rho(s))} = +\infty.$
    \end{itemize}
\end{enumerate}

Under the definition of our manual bridges, the only condition satisfied from above is $\lim_{t \downarrow 0} \zeta(t) = +\infty$. Therefore, manual bridges are not necessarily diffusion bridges.

\subsection{Diffusion bridges of \texorpdfstring{\citet{liu2023learning}}{}}
\citet{liu2023learning} provides a particular choice of diffusion bridges and shows they constitute a valid diffusion bridge. Here we re-state their proposed bridge expressions in our notation. Furthermore, we show their proposed bridge is equivalent to an optimal diffusion model for a data distribution of $\gU(\Omega)$, a uniform distribution on $\Omega$.

\paragraph{Diffusion bridges of \citet{liu2023learning}}
Consider a diffusion model as defined in \cref{sec:background:dm} with the forward process
\begin{equation}
    d\rvx_t = f(\rvx_t; t)\,dt + g(t)\,d\rvw,
\end{equation}
and conditional forward density 
\begin{equation}
    q_t(\rvx_t | \rvx_0) = \gN(\rvx_t; \alpha(t) \rvx_0, \sigma^2(t)\mI),
\end{equation}
where $\alpha(t)$ and $\sigma(t)$ are respectively the total scale and noise in the forward process.
As mentioned in \cref{sec:background:db}, diffusion bridges have the form of
\begin{equation}
    d \rvx_t = [\nu_\theta(\rvx_t; t) - g^2(t) \gB^\Omega(\rvx_t; t)] dt + g(t) d\bar{\rvw},
\end{equation}
where $\nu_\theta(\rvx_t; t) := f(\rvx_t; t) - g^2(t) \score_\theta(\rvx_t; t)$. \citet{liu2023learning} proposes the following form for the bridge update $\gB$:
\begin{equation}
    \gB^\Omega(\rvx_t; t) = \meanp{r^\Omega(\rvx_0 | \rvx_t)}{\nabla_{\rvx_t} \log r(\rvx_0 | \rvx_t)},
\end{equation}
where $r(\rvx_0|\rvx_t) := \gN(\rvx_0; \rvx_t / \alpha(t), \sigma^2(t)/ \alpha^2(t) \mI)$ and $r^\Omega(\rvx_0|\rvx_t) = \frac{r(\rvx_0 | \rvx_t) \1(\rvx_0 \in \Omega)}{Z(\rvx_t)}$ where $Z(\rvx_t)$ is the normalizing factor and is independent of $\rvx_0$.

\paragraph{Connection to a diffusion model for $\gU(\Omega)$}
\begin{lemma}\label{lemma:app:r-q-same-score-fn}
    For any $\rvx_0$ and $\rvx_t$ in $\R^d$ the following holds
    \begin{enumerate}
        \item $\nabla_{\rvx_t} \log r(\rvx_0 | \rvx_t) = \nabla_{\rvx_t} \log q(\rvx_t | \rvx_0)$,
        \item $\nabla_{\rvx_0} \log r(\rvx_0 | \rvx_t) = \nabla_{\rvx_0} \log q(\rvx_t | \rvx_0)$.
    \end{enumerate}
\end{lemma}
\begin{proof}
    We first derive $\nabla_{\rvx_t} \log r(\rvx_0 | \rvx_t)$:
    \begin{align}
        &r(\rvx_0 | \rvx_t) = \gN(\rvx_0; \rvx_t / \alpha(t), \sigma^2 / \alpha^2(t) \mI)\nonumber\\
        \Rightarrow &\nabla_{\rvx_t} \log r(\rvx_0 | \rvx_t) = \nabla_{\rvx_t} \frac{-(\rvx_0 - \frac{\rvx_t}{\alpha(t)})^\transpose(\rvx_0 - \frac{\rvx_t}{\alpha(t)})}{2\sigma^2(t) / \alpha^2(t)}
        = \frac{\rvx_0 - \frac{\rvx_t}{\alpha(t)}}{\sigma^2(t) / \alpha(t)} = \frac{\alpha(t) \rvx_0 - \rvx_t}{\sigma^2(t)}.
    \end{align}
    Then we derive $\nabla_{\rvx_t} \log q(\rvx_t | \rvx_0)$:
    \begin{align}
        &q(\rvx_t | \rvx_0) = \gN(\rvx_t; \alpha(t)\rvx_0, \sigma^2\mI)\nonumber\\
        \Rightarrow &\nabla_{\rvx_t} \log q(\rvx_t | \rvx_0) = \nabla_{\rvx_t} \frac{-(\rvx_t - \alpha(t)\rvx_0)^\transpose(\rvx_t - \alpha(t)\rvx_0)}{2\sigma^2(t)}
        = \frac{\alpha(t) \rvx_0 - \rvx_t}{\sigma^2(t)}.
    \end{align}
    Therefore, $\nabla_{\rvx_t} \log r(\rvx_0 | \rvx_t) = \nabla_{\rvx_t} \log q(\rvx_t | \rvx_0)$. Proof of the second part is similar.
\end{proof}
\begin{proposition} The bridge update of \citet{liu2023learning} is equivalent to an optimal diffusion model for $\gU(\Omega)$.
\end{proposition}
\begin{proof}
    Consider a diffusion process with the data distribution $q(\rvx_0) = \gU(\rvx_0; \Omega)$. Let's first look at the score function of $q(\rvx_0 | \rvx_t)$:
    \begin{align*}
        \nabla_{\rvx_0} \log q(\rvx_0 | \rvx_t) &= \nabla_{\rvx_0} \log q(\rvx_t | \rvx_0) + \nabla_{\rvx_0} \log q(\rvx_0)\\
        &\hspace{-1.1em}\overset{\text{\cref{lemma:app:r-q-same-score-fn}}}{=} \nabla_{\rvx_0} \log r(\rvx_0 | \rvx_t) + \nabla_{\rvx_0} \log q(\rvx_0)\\
        &= \nabla_{\rvx_0} \log r(\rvx_0 | \rvx_t) + \nabla_{\rvx_0} \log \1(\rvx_0 \in \Omega)\\
        &= \nabla_{\rvx_0} \log \left(r(\rvx_0 | \rvx_t) \1(\rvx_0 \in \Omega)\right)\\
        &= \nabla_{\rvx_0} \log \left(r^\Omega(\rvx_0 | \rvx_t) Z(\rvx_t)\right)\\
        &= \nabla_{\rvx_0} \log r^\Omega(\rvx_0 | \rvx_t)
    \end{align*}
    Therefore, $q(\rvx_0 | \rvx_t)$ and $r^\Omega(\rvx_0 | \rvx_t)$ have the same score function, hence they are equal. Furthermore, from \cref{lemma:app:r-q-same-score-fn}, we know $\nabla_{\rvx_0} \log r(\rvx_0 | \rvx_t) = \nabla_{\rvx_0} \log q(\rvx_t | \rvx_0)$. Consequently,
    \begin{align}
        \gB^\Omega(\rvx_t; t) &= \meanp{r^\Omega(\rvx_0 | \rvx_t)}{\nabla_{\rvx_t} \log r(\rvx_0 | \rvx_t)}
        = \meanp{q(\rvx_0 | \rvx_t)}{\nabla_{\rvx_t} \log q(\rvx_t | \rvx_0)}.
    \end{align}
    From \citet{vincent2011connection} we know
    \begin{equation}
        \nabla_{\rvx_t} \log q(\rvx_t) = \meanp{q(\rvx_0 | \rvx_t)}{\nabla_{\rvx_t} q(\rvx_t | \rvx_0)}.
    \end{equation}
    Therefore, $\gB^\Omega(\rvx_t; t)$ proposed by \citet{liu2023learning} is equivalent to the score functions of an optimal diffusion model with the data distribution $q(\rvx_0)$ equal to a uniform distribution on the constraint set $\Omega$.
\end{proof}

\section{Proofs}
We mentioned in the main text in \cref{sec:method:manual-bridge} that manually bridged models define score functions of distributions $p^\Omega(\rvx; t)$ that anneal to $p^\Omega(\rvx; 0) \propto p(\rvx; 0) \1_\Omega(\rvx)$. Here we formally prove this.

As stated before, the manually bridged models correspond to distributions of the from
\begin{align}
    p^\Omega(\rvx; t) = \frac{p(\rvx; t) \exp(-\gamma(t) \distancefn^\Omega(\rvx; t))}{Z(t)},
\end{align}
where $Z(t) = \int p(\rvx) \exp(-\gamma(t) \distancefn^\Omega(\rvx; t)) d\rvx$ is the normalizing constant. Since the integrand is non-negative,
\begin{equation*}
    Z(t) \geq \int_\Omega p(\rvx; t) \exp(-\gamma(t) \distancefn^\Omega(\rvx; t)) d\rvx = \int_\Omega p(\rvx; t) d\rvx,
\end{equation*}
where the last equality is due to $\distancefn^\Omega$ being zero on $\Omega$. Define $\alpha := \int_\Omega p(\rvx; t) d\rvx > 0$. Therefore, $Z(t)$ is bounded below by the constant $\alpha$. Hence,
\begin{align}
    p^\Omega(\rvx; t) \leq \frac{1}{\alpha} p(\rvx; t) \exp(-\gamma(t) \distancefn^\Omega(\rvx; t)).
    \label{eq:app:density-bound}
\end{align}

\begin{lemma}
    \label{lemma:app:converge-zero}
    For any arbitrary $\rvx \notin \Omega$, $\lim_{t \downarrow 0} p(\rvx; t) \exp(-\gamma(t)\distancefn^\Omega(\rvx)) = 0$.
\end{lemma}
\begin{proof}
    Consider any arbitrary $\rvx \notin \Omega$ and $\epsilon >0$. From \cref{prop:mbm-sequence-of-distributions}, $p$ is finite. Denote an upper bound of $p$ by $M$ i.e., $p(\rvx, t) \leq M$ for any $t$.
     Furthermore, since $\rvx \notin \Omega$ and $\distancefn$ is an $\Omega$-distance function, $\distancefn^\Omega(\rvx; 0) = C$ for some $C > 0$. Since $\distancefn^\Omega$ is continuous in $t$, $\lim_{t \downarrow 0} \distancefn(\rvx; t) = C$. Consider an arbitrary $0 < c < C$,
     \begin{equation}
         (\exists \delta_1 > 0) \; (t < \delta_1 \Rightarrow \distancefn^\Omega(\rvx; t) > c) \label{app:proofs:convergence:distancefn-lim}
     \end{equation}
    Also from \cref{def:omega-distance-fn}, $\lim_{t \downarrow 0} \gamma(t) = \infty$. Therefore, for a given $\epsilon$,
    \begin{equation}
        (\exists \delta > 0) \; t < \delta_2 \Rightarrow \gamma(t) > -\frac{\log (\epsilon / M) }{c}
    \end{equation}
    
    Let $\tau := \min\{\delta_1, \delta_2\}$. For all $t < \tau$ we have
    \begin{align}
        &\gamma(t) > - \frac{\log(\epsilon / M)}{c}
        \Rightarrow -\gamma(t) c < \log(\epsilon/M)\nonumber\\
        \Rightarrow & \exp(-\gamma(t)c) < \frac{\epsilon}{M}
    \end{align}
    Remember that $\gamma$ is non-negative and $c < \distancefn^\Omega(\rvx; t)$ (from \cref{app:proofs:convergence:distancefn-lim}). Hence,
    \begin{equation}
        \exp(-\gamma(t) \distancefn^\Omega(\rvx; t)) < \exp(-\gamma(t)c) < \frac{\epsilon}{M}
    \end{equation}
    Here, if $p(\rvx, t)$ is equal to zero, $p(\rvx; t) \exp(-\gamma(t)\distancefn^\Omega(\rvx; t))$ is also equal to zero. Otherwise,
    \begin{align}
        \exp(-\gamma(t)\distancefn^\Omega(\rvx; t)) < \frac{\epsilon}{M} \leq \frac{\epsilon}{p(\rvx, t)} \Rightarrow p(\rvx; t) \exp(-\gamma(t)\distancefn^\Omega(\rvx; t)) < \epsilon.
    \end{align}
    Therefore, in both cases $p(\rvx; t) \exp(-\gamma(t)\distancefn^\Omega(\rvx; t)) < \epsilon$ which concludes the proof.
\end{proof}
Following \cref{lemma:app:converge-zero} and using \cref{eq:app:density-bound}, for all $\rvx \notin \Omega$, we have $\lim_{t \downarrow 0} p^\Omega(\rvx; t) = 0$. Following the continuity of $p^\Omega(\rvx; t)$ in $t$, we conclude 
\begin{equation}
    \forall \rvx \notin \Omega, \; p^\Omega(\rvx; 0) = 0.
\end{equation}
Further, since for any $\rvx \in \Omega$, by definition $\distancefn^\Omega(\rvx; 0) = 0$,
\begin{equation}
    \forall \rvx \in \Omega, \; p^\Omega(\rvx; 0) = p(\rvx; 0).
\end{equation}
Therefore, $p(\rvx; 0)$ is an $\Omega$-constrained distribution, meaning it will only place mass on $\Omega$.

\section{Image watermarking experiment}
Here we present our image watermarking experiment. We follow the experimental setup of \citep{liu2024mirror} where an image is considered ``watermarked'' if it lies within an orthonormal polytope $\Omega := \{\rvx \in \R^d: c_i < a_i^\transpose \rvx < b_i, \forall i\}$. Here, $a_i \in \R^d$ and $b_i$ and $c_i$ are scalars. The parameters $\{a_i, b_i, c_i\}_{i=1}^m$ are private tokens only visible to users.
Our implementation is based on the codebase of \citet{liu2024mirror}\footnote{\url{https://github.com/ghliu/mdm/}} with some modifications explained below.

\paragraph{Watermark parameters} We use 100 random orthonormal vectors $a_i$ that were generated and provided in the codebase of \citet{liu2024mirror}. The other parameters are $b_i = -0.9$ and $c_i = 0.9$.

\paragraph{Dataset} We use a watermarked version of the AFHQv2 dataset~\citep{choi2020stargan} for this experiment. Remember that our problem setup requires all the training data to fall within the constraint set $\Omega$. This is also the case for the Mirror Diffusion Models (MDM) \citet{liu2023learning}. Given a set of watermark parameters, we first watermark all the images in the training dataset. We then train our models and baselines on this watermarked dataset.
Following \citet{liu2024mirror}, we watermark an image by going through the $a_i$ parameters (orthonormal vectors) of the watermark and if the condition of $b_i < a_i^\transpose \rvx < c_i$ is not satisfied, we shift the features $\rvx$ by a random amount along the vector $a_i$ such that the constraint is satisfied. More formally, $\rvx' \leftarrow \rvx + \delta a_i$ where $\delta$ is sampled from a uniform distribution with bounds derived according to $a_i^\transpose\rvx$ and the bounds $b_i$ and $c_i$. As $a_i$s are orthogonal to each other, this modification would not violate the other constraints.

\paragraph{Evaluation metrics}
We report the following four metrics in this experiment:
\begin{enumerate}
    \item Infraction rate: The percentage of images generated by the model that violate the watermark constraints.
    \item Infraction loss: The average value of the distance function $\distancefn^\Omega$ of the images generated from the model.
    \item FID: The Fréchet Inception Distance~\citep{heusel2017gans} between the generated images and the watermarked dataset. Note that the watermarked dataset is not the original AFHQv2 dataset but the watermarked version of it. This metric measures the perceptual similarity of the images generated from the model and the ones it was trained on.
    \item r-ELBO: The reweighted ELBO of the dataset assigned by model. This is effectively the training loss of the model.
\end{enumerate}
The first three metrics (infraction rate, infraction loss and FID) are computed on $N$ samples generated by the model where $N=15,804$ is the size of the training dataset.

\paragraph{Baselines} We compare te manual bridges under the three architectures \texttt{C-arch}, \texttt{DB-arch} and \texttt{MBM-arch} with the following baselines:
\begin{itemize}
    \item EDM-pretrained: a pretrained diffusion model from \citet{karras2022elucidating}. This model was trained on the original, non-watermarked AFHVQv2 dataset.
    \item EDM-finetuned: the same model as above but fine-tuned on the watermarked dataset.
    \item MDM: the Mirror Diffusion Model (MDM) from \citet{liu2024mirror}. MDM is restricted to convex constraints and requires defining a bijective mapping between the constraint set $\Omega$ and Euclidean space $\R^d$. This mapping is called a mirror map and is denoted by $(\nabla\phi, \nabla \phi^*)$ where $\nabla \phi^*$ is the inverse of $\nabla \phi$. Given a dataset with data points $\rvx \in \Omega$, MDM first applies $\nabla \phi$ to get $\rvy = \nabla\phi(\rvx) \in \R^d$ and trains a standard diffusion on the resulting dataset. Once trained, the samples generated by the model are transformed back to the constrained space by applying $\nabla \phi^*$. As a result, samples from MDM are guaranteed to lie within the constraint set $\Omega$.
    \item MDM-proj: the same model as EDM-finetuned, but with a projection operator applied to the generated images to ensure they fall within the constraint set.
\end{itemize}

\paragraph{Manual bridge definition} Here we explain different ingredients of manual bridges implemented. The ingredients are the distance function $\distancefn^\Omega$, the time-dependent scaling $\gamma$ and the conditioning signal used in $\texttt{C-arch}$ and $\texttt{MBM-arch}$.
\begin{itemize}
    \item \underline{Distance function}: We define a distance function $\distancefn^\Omega(\rvx; t) = \frac{1}{2} \norm{\mathrm{Proj}(\rvx) - \rvx}^2$. However, to avoid computational cost of backpropagating through the projection operator, we modify it to $\distancefn^\Omega(\rvx; t) = \frac{1}{2} \norm{\overline{\mathrm{Proj}(\rvx)} - \rvx}^2$ where the bar on $\overline{\mathrm{Proj}(\rvx)}$ denotes the stop-grad operation. Note $\distancefn^\Omega$ remains a distance function with this modification.
    \item \underline{Scale}: Our choice of scale function is $\gamma(t) = \frac{1}{\sigma^2(t)}\frac{\sigma_{\mathrm{data}}^2}{\sigma^2(t) + \sigma_{\mathrm{data}}^2}$. This satisfies the required conditions of $\gamma(T) \approx 0$ and $\lim_{t \rightarrow 0} \gamma(t) = \infty$.
    \item \underline{Conditioning}: We apply a separate scaling to the signal we pass to the model to condition on (in our \texttt{MBM-arch} as well as \texttt{C-arch}) in order to avoid its scale becoming large for large noise levels. The conditional signal is equal to $-\gamma'(t) \nabla_{\rvx_t}(\distancefn^\Omega(\rvx_t; t))$ where $\gamma'(t) = \frac{\sigma_\mathrm{data}}{\sqrt{\sigma^2(t) + \sigma_\mathrm{data}^2}}$.
    This is because $\Var[{\rvx_t}] = \sigma^2(t) + \sigma_\mathrm{data}^2$ in the variance-exploding diffusion process of EDM. Furthermore, we add a channel to the conditioning signal filled with $\1_{x \neq 0}(b(\rvx_t; t))$ where $\1_{x \neq 0}$ is an element-wise binary function that is only equal to zero when its input is zero. This modification empirically boosted the performance of our models.
\end{itemize}

\paragraph{Training} all the models and baselines (apart from EDM-pretrained which does not involve further training), are trained for $200,000$ iterations with a batch size of $256$. We use Adam optimizer with a learning rate of $2 \times 10^{-4}$. All the models initialized from the weights of EDM-pretrained.

\paragraph{Sampling} We used EDM's sampler with Heun correction and the parameters of $S_{\mathrm{churn}}=10$, $\sigma_{\mathrm{max}} = 80$ and $\sigma_{\mathrm{max}} = 10^{-5}$.

\paragraph{Results} \cref{tab:watermark-afhqv2-experiment} reports ours results of this experiment. We observe that MDM and MDM-proj achieve zero infraction rate because they both include a transformation operator in the end that places the generated sample in the constraint set. Intuitively, MDM-proj shifts the distribution since all the constrained-violating samples (which is more than 91\% of them as the samples are generated by EDM-finetuned) are moved to the boundary, resulting in an overpopulation of samples at the boundary. MDM, on the other hand, requires the mirror maps that are not available for complex constraints. In particular, it requires convex constraints. While the mirror maps are available in this experiment, in a more general setting such as that of \cref{sec:exp:ic}. Nonetheless, MDM suffers from a poor performance in terms of FID, suggesting a distribution shift. Manual bridges with \texttt{MBM-arch} achieves almost-zero infraction rate with the best FID and r-ELBO comparable to the best of baselines.

\begin{table*}[t] 
  \centering
  \begin{tabular}{lllll}
    \toprule
    \centering
    Method     & Infraction ($\%$)  $\downarrow$  & Inf. loss  $\downarrow$  & FID $\downarrow$ & r-ELBO ($\times 10^2$)$\uparrow$\\
    \midrule
    EDM-pretrained~\cite{karras2022elucidating} & $91.44 \pm 0.25$ & $0.35$ & $5.24 \pm 0.06$ & \underline{$-15.95 \pm 0.09$}\\
    EDM-finetuned & $91.52 \pm 0.39$ & $0.36$ & $2.95 \pm 0.09$ & $\mathbf{-15.78 \pm 0.07}$\\
    \midrule
    MDM~\citep{liu2024mirror} & $0$ & $0$ & $3.73 \pm 0.10$ & --\\
    MDM-proj~\citep{liu2024mirror} & $0$ & $0$ & $\underline{2.90 \pm 0.09}$ & --\\
    \midrule
    Manual bridge with \texttt{C-arch} & $90.74 \pm 0.17$ & $0.31$ & \underline{$2.77 \pm 0.07$} & $\mathbf{-15.80 \pm 0.05}$\\
    Manual bridge with \texttt{DB-arch} & $0$ & $0$ & $\underline{2.81 \pm 0.05}$ & $-16.54 \pm 0.10$\\
    Manual bridge with \texttt{MBM-arch} (ours) & $0.02 \pm 0.00$ & $3 \times 10^{-12}$ & $\mathbf{2.65 \pm 0.07}$ & $\mathbf{-15.83 \pm 0.03}$\\
    \bottomrule
  \end{tabular}
  \caption{Results for the image watermarking experiment on the AFHQv2 dataset.}
  \label{tab:watermark-afhqv2-experiment}
\end{table*}

\section{Additional Traffic Scene Generation Results}
In this section we provide more results from the traffic scene generation experiment. \cref{fig:app:additional:plots} shows the collision and offroad loss over the course of training of different models. Note that these models are all trained \textit{from scratch} as opposed to the results in the main text that the models were fine-tuned. Also note that collision and offroad loss are different than collision and offroad rates reported in the main text. The metrics here measure ``how bad'' the infractions are. A collision loss of zero corresponds to a collision rate of zero, and similarly for offroad. The trend we observe here is similar to that of the main text.

In \cref{fig:app:additional:samples} we provide more samples from different models, similar to the top row of \cref{fig:banner}.

\begin{figure*}[!htp]
    \centering
    \begin{subfigure}[t]{0.9\textwidth}
        \includegraphics[width=\textwidth]{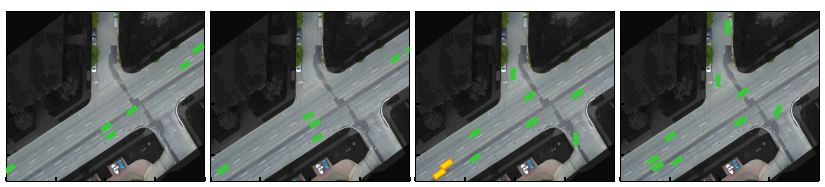}
    \end{subfigure}
    
    \begin{subfigure}[t]{0.9\textwidth}
        \includegraphics[width=\textwidth]{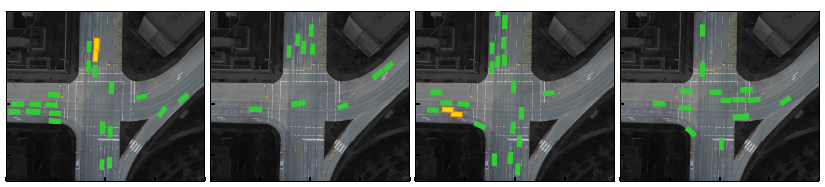}
    \end{subfigure}

    \caption{A few more visualizations from Traffic Scene Generation experiment. From left to right: Standard diffusion, Bridge model (Offset), Conditional diffusion and \method{}. As traffic scenes become crowded, the occurrences of infraction show up in the samples generated from standard diffusion and condition diffusion models. While no infraction happens in the samples generated from Bridge model (Offset), the vehicle orientation and distances between vehicles are not convincingly realistic. Samples from \method{} appears more realistic and natural.}
    \label{fig:app:additional:samples}
\end{figure*}